\documentclass[twoside,11pt]{article}

\usepackage{jmlr2e1}

\usepackage{amsmath, amsfonts, amssymb, amsxtra}
\usepackage{mathrsfs, mathtools}
\usepackage{bbm}
\usepackage{enumitem}
\usepackage{natbib}
\usepackage{graphicx}
\usepackage[usenames, dvipsnames]{xcolor}
\newcommand{\eps}{\varepsilon}
\newcommand{\KL}{\mathsf{KL}}

\numberwithin{equation}{section}

\newcommand{\DS}{\displaystyle}

\newcommand{\cB}{\mathcal{B}}

\newcommand{\cE}{\mathcal{E}}

\newcommand{\cI}{\mathcal{I}}

\newcommand{\cM}{\mathcal{M}}
\newcommand{\cN}{\mathcal{N}}

\newcommand{\cP}{\mathcal{P}}

\newcommand{\cS}{\mathcal{S}}

\newcommand{\bone}{\mathbf{1}}

\newcommand{\R}{\mathbbm{R}}

\newcommand{\p}{\mathbbm{P}}
\newcommand{\E}{\mathbbm{E}}
\newcommand{\Z}{\mathbbm{Z}}
\newcommand{\1}{\mathbbm{1}}

\newcommand{\labelsubscript}[2]{$#1 _ #2$}

\newcommand{\Ber}{\mathsf{Ber}}
\newcommand{\Bin}{\mathsf{Bin}}

\newcommand{\id}{\mathsf{id}}
\newcommand{\inver}{d_{\mathsf{KT}}}
\newcommand{\defn}{:\,=}
\newcommand{\cond}{\,|\,}
\newcommand{\MS}{{\sf MS}}
\newcommand{\Mns}{\cM_n(\lambda)}
\newcommand{\Mnss}{M^*_n(\lambda)}

\jmlrheading{}{}{}{}{}{}{Cheng Mao, Jonathan Weed and Philippe Rigollet}

\ShortHeadings{Noisy Sorting}{Mao, Weed and Rigollet}
\firstpageno{1}

\begin{document}

\title{Minimax Rates and Efficient Algorithms for Noisy Sorting}

\author{\name Cheng Mao \email maocheng@mit.edu \\
\name Jonathan Weed \email jweed@mit.edu \\
\name Philippe Rigollet \email rigollet@mit.edu \\
\addr Department of Mathematics\\
Massachusetts Institute of Technology\\
Cambridge, MA 02139-4307, USA
}

\editor{}

\maketitle

\begin{abstract}
There has been a recent surge of interest in studying permutation-based models for ranking from pairwise comparison data. Despite being structurally richer and more robust than parametric ranking models, permutation-based models are less well understood statistically and generally lack efficient learning algorithms. In this work, we study a prototype of permutation-based ranking models, namely, the noisy sorting model. We establish the optimal rates of learning the model under two sampling procedures. Furthermore, we provide a fast algorithm to achieve near-optimal rates if the observations are sampled independently. Along the way, we discover properties of the symmetric group which are of theoretical interest.
\end{abstract}


\section{Introduction} \label{sec:intro}

Pairwise comparison data is frequently observed in various domains, including recommender systems, website ranking, voting and social choice \citep[see, e.g.][]{BalMakRic10,Dwoetal01,Liu09,You88,CapNal91}. For these applications, it is of significant interest to produce a suitable ranking of the items by aggregating the outcomes of pairwise comparisons. The general problem of interest can be stated as follows. Suppose there are $n$ items to be compared and an underlying matrix $P$ of probability parameters, each entry $P_{i,j}$ of which represents the probability that item $i$ beats item $j$ if they are compared. Hence we have $P_{j,i} = 1-P_{i,j}$ and the event that item $i$ beats item $j$ in a comparison can be viewed as a Bernoulli random variable with probability $P_{i,j}$. Observing the outcomes of $N$ independent pairwise comparisons, we aim to estimate the absolute ranking of the items.

For the sake of consistency,
%To achieve consistence, 
one needs of course to impose some structure on the matrix $P=\{P_{i,j}\}_{1\le i,j\le n}$. These structural assumptions are traditionally split between \emph{parametric} and \emph{nonparametric} ones. Classical \emph{parametric models} include the Bradley-Terry-Luce model \citep{BraTer52,Luc59} and the Thurstone model~\citep{Thu27}. These models can be recast as log-linear models, which enables the use of the statistical and computational machinery of maximum likelihood estimation in generalized linear models~\citep{Hun04,NegOhSha12,RajAga14,HajOhXu14,Shaetal15,NegOhSha16,Negetal17}.

To allow richer structures on $P$ beyond the scope of parametric models, \emph{permutation-based models} such as the \emph{noisy sorting model} \citep{BraMos08,BraMos09} and the \emph{strong stochastic transitivity} (SST) model \citep{Cha15,ShaBalGunWai17} have recently become more prevalent. These models only require shape constraints on the matrix $P$ and are typically called \emph{nonparametric}. In these models, the underlying ranking of items is determined by an unknown permutation $\pi^*$, and, additionally, the comparison probabilities are assumed to have a bi-isotonic structure when the items are aligned according to $\pi^*$. While permutation-based models provide ordering structures that are not captured by parametric models~\citep{Aga16,ShaBalGunWai17}, they introduce both statistical and computational barriers for estimation of the underlying ranking. These barriers are mainly due to the complexity of the discrete set of permutations. On the one hand, the complexity of the set of permutations is not well understood~\citep[see the discussion following Theorem~8 in][]{ColDal16}, which leads to logarithmic gaps in the current statistical bounds for permutation-based models. On the other hand, it is computationally challenging to optimize over the set of permutations, so current algorithms either sacrifice nontrivial statistical performance or have impractical time complexity. In this work, we aim to address both questions for the noisy sorting model.

In practice, it is unlikely that all the items are compared to each other. To account for this limitation, a widely used scheme consists in assuming that that each pairwise comparison is observed with probability $p \in (0,1]$ \citep[see, e.g.][]{Cha15,ShaBalGunWai17}. In addition to this model of missing comparisons, we study the model where $N$ pairwise comparisons are sampled uniformly at random from the $\binom{n}{2}$ pairs, with replacement and independent of each other. It turns out that sampling with and without replacement yields the same rate of estimation up to a constant when the expected numbers of observations coincide.

\paragraph{Our contributions.} 
We focus on the noisy sorting model with partial observations, under which a stronger item wins a comparison against a weaker item with probability at least $\frac 12 + \lambda$ where $\lambda \in (0,\frac 12)$.
%We focus on the noisy sorting model with partial observations, for which a stronger item wins a comparison against a weaker item with a fixed probability $\frac 12 + \lambda$ where $\lambda \in (0,\frac 12)$.
%$\frac 12 + \lambda^*$ where $\lambda^* \in (0,\frac 12)$. Besides the probability parameter $\lambda^*$, the unknown permutation is the only parameter of the model. Hence noisy sorting is a suitable prototype of permutation-based models that we aim to understand. 
%Aside from the noisy comparison model, we propose two sampling models determining which pairs are selected from the $\binom n2$ pairs of items. Regarding 
For sampling both with and without replacement, we establish the minimax rate of learning the underlying permutation. In particular, the rate does not involve a logarithmic term, and we explain this phenomenon through a careful analysis of the metric entropy of the set of permutations equipped with  the Kendall tau distance, which is of independent theoretical interest.

Moreover, we propose a multistage sorting algorithm that has time complexity $\tilde O(n^2)$. For the sampling with replacement model, we prove a theoretical guarantee on the performance of the multistage sorting algorithm, which differs from the minimax rate by only a polylogarithmic factor. In addition, the algorithm is demonstrated to perform similarly for both sampling models using simulated examples.

\paragraph{Related work.}

The noisy sorting model was proposed by \citet{BraMos08}. In the original paper, the optimal rate of estimation achieved by the maximum likelihood estimator (MLE) is established, and an algorithm with time complexity $O(n^C)$ is shown to find the MLE with high probability in the case of full observations\footnote{If the algorithm is allowed to actively choose the pairs to be compared, the sample complexity can be reduced to $O(n\log n)$. However, in the passive setting which we adopt throughout this work, the algorithm still needs $\Theta(n^2)$ pairwise comparisons.}, where $C = C(\lambda)$ is a large unknown constant. Moreover, their algorithm does not have a polynomial running time if only $o(n^2)$ random pairwise comparisons are observed.
Our work generalizes the optimal rate to the partial observation settings by studying a variant of the MLE for the upper bound. In the  model of sampling with replacement, our fast multistage sorting algorithm provably achieves near-optimal rate of estimation.
Since finding the MLE for the noisy sorting model is an instance of the NP-hard feedback arc set problem \citep{Alo06,KenSch07,AilChaNew08,BraMos08}, our results indicate that, despite the NP-hardness of the worst-case problem, it is still possible to achieve (near-)optimal rates for the average-case statistical setting in polynomial time.

%Therefore, achieving (near-)optimal rates for average-case statistical models is not necessarily hindered by NP-hard optimization problems.

The SST model generalizes the noisy sorting model, and minimax rates in the SST model have been studied by \cite{ShaBalGunWai17}. However, the upper bound specialized to noisy sorting contains an extra logarithmic factor, which this work shows to be unnecessary.
%is unnecessary as exhibited by the current work.
Moreover, the lower bound there is based on noisy sorting models with $\lambda$ shrinking to zero as $n \to \infty$, while we establish a matching lower bound at any fixed $\lambda$. In addition, algorithms of \cite{WauJorJoj13,ShaBalGunWai17,ChaMuk16} are all statistically suboptimal for the noisy sorting model. This is partially addressed by our multistage sorting algorithm as discussed above.

In fact, both with- and without-replacement sampling models discussed in this paper are restrictive for applications where the set of observed comparisons is subject to certain structural constraints~\citep{HajOhXu14,Shaetal15,Negetal17,Panetal17}. Obtaining sharper rates of estimation for these more complex sampling models is of significant interest but is beyond the scope of the current work.

Finally, we mention a few other lines of related work. Besides permutation-based models, low-rank structures have also been proposed by \cite{RajAga16} to generalize classical parametric models. Moreover, there is an extensive literature on active ranking from pairwise comparisons  \citep[see, e.g.,][and references therein]{JamNow11,Hecetal16,Agaetal17}, where the pairs to be compared are chosen actively and in a sequential fashion by the learner. The sequential nature of the models greatly reduces sample complexity, so we do not compare our results for passive observations to the literature on active learning. However, it is interesting to note that our multistage sorting algorithm is reminiscent of active algorithms, because it uses different batches of samples for different stages. Thus active learning algorithms could potentially be useful even for passive sampling models.
%Beyond ranking problems, permutations are ubiquitous in a variety of disciplines. 

\paragraph{Organization.}

The noisy sorting model together with the two sampling models is formalized in Section~\ref{sec:setup}. In Section~\ref{sec:result}, we present our main results, the minimax rate of estimation for the latent permutation and the near-optimal rate achieved by an efficient multistage sorting algorithm. To complement our theoretical findings, we inspect the empirical performance of the multistage sorting algorithm on numerical examples in Section~\ref{sec:sim}.  Section~\ref{sec:inver} is devoted to the study of the set of permutations equipped with the Kendall tau distance. Proofs of the main results are provided in Section~\ref{sec:proof}. We discuss directions for future research in Section~\ref{sec:dis}.

\paragraph{Notation.}
For a positive integer $n$, let $[n]=\{1,\ldots, n\}$. For a finite set $S$, we denote its cardinality by $|S|$. Given $a, b \in \R$, let $a\wedge b=\min(a,b)$ and $a\vee b=\max(a,b)$.
%We use $C, C_1, C_2, \dots$ to denote sufficiently large positive constants and $c, c_1, c_2, \dots$ to denote sufficiently small positive constants that may change at each appearance. 
We use $C$ and $c$, possibly with subscripts, to denote universal positive constants that may change at each appearance. 
For two sequences $\{u_n\}_{n=1}^\infty$ and $\{v_n\}_{n=1}^\infty$, we write $u_n \lesssim v_n$ if there exists a universal constant $C>0$ such that $u_n \le C v_n$ for all $n$. We define the relation $u_n \gtrsim v_n$ analogously, and write $u_n \asymp v_n$ if both $u_n \lesssim v_n$ and $u_n \gtrsim v_n$ hold.
%For a matrix $M \in {\R}^{n \times n}$, 
%let $\|M\|_F$ denote its Frobenius norm, and 
%let $M_{i,\cdot}$ denote its $i$-th row and $M_{\cdot,j}$ denote its $j$-th column.
Let $\mathfrak S_n$ denote the symmetric group on $[n]$, i.e., the set of permutations $\pi:[n] \to [n]$. 
%For $\pi \in \mathfrak S_n$, define $M_\pi \in \R^{n\times n}$ by $(M_\pi)_{i, j} = M_{\pi(i), \pi(j)}$ for $i ,j \in [n]$. 
%For $v \in \R^n$, define $v_\pi \in \R^n$ by $(v_\pi)_i = v_{\pi(i)}$ for $i \in [n]$.

\section{Problem formulation} \label{sec:setup}

The noisy sorting model can be formulated as follows. Fix an unknown permutation $\pi^* \in \mathfrak S_n$ which determines the underlying order of $n$ items. More precisely, $\pi^*$ orders the items from the weakest to the strongest, so that item $i$ is the $\pi^*(i)$-th weakest among the $n$ items. For a fixed $\lambda \in (0,1/2)$, 
we define a class of matrices
\begin{align*}
\Mns=\Big\{ M \in [0,1]^{n \times n} \,:\,  \, M_{i,i}= \frac 12 \,, \ M_{i, j}\ge \frac12+\lambda\ \text{if}\ i>j\,, \ M_{i, j}\le \frac12-\lambda\ \text{if}\ i<j \Big\} \,,
\end{align*}
%\begin{align*}
%\Mns=\Big\{ \tfrac 12 \bone_n \bone_n^\top + \Delta \,:\, & \, \Delta \in \R^{n\times n}, \, \Delta_{i,i}=0 \text{ for } i \in [n] , \text{ and } \\
%&\lambda \le |\Delta_{i,j}| \le \tfrac 12 \text{ and } \Delta_{i,j} (i-j) > 0  \text{ for } i \ne j \Big\} \,,
%\end{align*}
where $\bone_n$ is the $n$-dimensional all-ones vector.
%In other words, if $M \in \Mns$, then all the entries below the diagonal of $M$ are in $[\frac 12 + \lambda, 1]$, while all the entries above the diagonal are in $[0, \frac 12 - \lambda]$. 
In addition, we define a special matrix $\Mnss \in \Mns$ by
$$
[\Mnss]_{i,j} =
\begin{cases}
1/2+\lambda & \text{ if } i > j \,, \\
1/2-\lambda & \text{ if } i < j \,, \\
1/2 & \text{ if } i = j \,.
\end{cases} 
$$
Note that $\Mnss$ satisfies strong stochastic transitivity but other matrices $M \in \Mns$ may not. Though this observation plays a crucial role in the design of efficient algorithms, our statistical results hold for general matrices in $\Mns$.

To model pairwise comparisons, fix  $M \in \Mns$ and let $M_{\pi^*(i), \pi^*(j)}$ denote the probability that items $i$ beats item $j$ when they are compared\footnote{The diagonal entries of $M$ are inessential in the model as an item is not compared to itself, and they are set to $1/2$ only for concreteness.}, so that a stronger item beats a weaker item with probability at least $\frac 12 + \lambda$. As a result, $\lambda$ captures the signal-to-noise ratio of our problem and our minimax results explicitly capture the dependence in this key parameter.

%In terms of the pairwise comparison model, suppose that we have an underlying matrix $M \in \Mns$, whose entry $M_{\pi^*(i), \pi^*(j)}$ represents the probability that items $i$ beats item $j$ when they are compared\footnote{The diagonal entries of $M$ are inessential in the model as an item is not compared to itself, and they are set to $1/2$ only for concreteness.}, so that a stronger item beats a weaker item with probability at least $\frac 12 + \lambda$. Therefore, $\lambda$ can be viewed as the signal strength in the model.

%\ndpr{What do you guys think of writing a class of matrices: 
%And  by arguing that the null model is $ M=(1/2)\bone_n\bone_n^\top$? I think it would give a smoother intro to the model and allow for a more general setup (though artificially). This is primarily in response to the way the Tsybakov results were framed in his talk yesterday (and also at the beginning of his class). It also covers us against critics that the model is quite restrictive. We could then write minimax result of the form:
%$$
%\inf_{\hat \pi} \sup_{M, \pi^*} \p^*(\inver \dots) \asymp \dots
%$$
%Feel free to dislike. This is merely cosmetics}

\subsection{Sampling models}
In the noisy sorting model, suppose that for each (unordered) pair $(i, j)$ with $i \ne j$, we observe the outcomes of $N_{i,j} (= N_{j,i})$ comparisons between them, and item $i$ wins a comparison against item $j$ with probability $M_{\pi^*(i), \pi^*(j)}$ independently. The set $\{N_{i,j}\}_{i < j}$ of $\binom{n}{2}$ nonnegative integers is determined by certain sampling models described below. We allow $N_{i,j}$ to be zero, which means that $i$ and $j$ are not compared. 
We collect sufficient statistics into a matrix $A \in \R^{n \times n}$ consisting of outcomes of pairwise comparisons, by defining $A_{i,j}$ to be the number of times item $i$ beats item $j$ among the $N_{i,j}$ comparisons between $i$ and $j$. In particular, we have $A_{i,j} + A_{j,i} = N_{i,j} = N_{j,i}$ for $i \ne j$ and $A_{i,i} = 0$. Our goal is to aggregate the results of pairwise comparisons to estimate $\pi^*$, the underlying order of items.

%Moreover, we construct a matrix $A \in \R^{n\times n}$ consisting of outcomes of pairwise comparisons, by defining $A_{i,j}$ to be the number of times item $i$ beats item $j$ in the $N_{i,j}$ comparisons. In particular, we have $A_{i,j} + A_{j,i} = N_{i,j} = N_{j,i}$ for $i \ne j$ and $A_{i,i} = 0$. Our goal is to aggregate the results of pairwise comparisons to estimate $\pi^*$, the underlying order of items.

In the full observation setup of \cite{BraMos08}, we have $N_{i,j} = 1$ for each pair $(i,j)$ and the total number of observations is $N \defn \sum_{i<j} N_{i,j} = \binom{n}{2}$. Instead, we are interested here in the regime where the total number of observations $N$ is much smaller than $\binom{n}{2}$. We study the following two sampling models in this work:

\begin{enumerate}[label=(\labelsubscript{O}{\arabic*})]
\item \label{model:1} \emph{Sampling without replacement.}
In this sampling model, instead of observing all the pairwise comparisons, we observe each pair with probability $p \in (0, 1]$ independently. Hence each $N_{i,j}\sim\Ber(p)$ is a Bernoulli random variable with parameter $p$, and in expectation we have $p \binom{n}{2}$ observations in total.

\item \label{model:2} \emph{Sampling with replacement.}
We observe $N$ pairwise comparisons between the items, sampled uniformly and independently with replacement from the $\binom{n}{2}$ pairs.
%Suppose that we observe $N$ pairwise comparisons between the items, sampled uniformly independently with replacement from the $\binom{n}{2}$ pairs.
\end{enumerate}

\noindent In the sequel, we study the noisy sorting model with either of the above two sampling models. In particular, the minimax rates of estimating $\pi^*$ coincide for the two sampling models if $p \binom{n}{2} \asymp N$, i.e., if the expected number of observations are of the same order.

\subsection{Measures of performance}

Having discussed the sampling and comparison models, we turn to the distance used to measure the difference between the underlying permutation $\pi^*$ and an estimated permutation $\hat \pi$.
Among various distances defined on the symmetric group, we consider primarily the \emph{Kendall tau distance}, i.e., the number of \emph{inversions} (or discordant pairs) between permutations, defined as
\[ \inver (\pi, \sigma) = \sum_{(i,j):\sigma(i) < \sigma(j)} \1 \big( \pi(i) > \pi(j) \big) \]
for $\pi, \sigma \in \mathfrak S_n$.
%If $\sigma$ is the identity $\mathsf{id}$ on $[n]$, we also write $\inver (\pi) = \inver (\pi, \mathsf{id})$ and call $\inver(\pi)$ the number of inversions of $\pi$ for simplicity. 
Note that $0 \le \inver(\pi, \sigma) \le \binom{n}{2}$. The Kendall tau distance between two permutations is a natural metric on $\mathfrak S_n$, and it is equal to the minimum number of adjacent transpositions required to change from one permutation to another~\citep{Knu98}. 
%An affine transformation of this quantity is known as Kendall rank correlation coefficient. 
A closely related distance on $\mathfrak S_n$ is the $\ell_1$-distance, also known as Spearman's footrule, defined as
\[ \|\pi - \sigma\|_1 = \sum_{i=1}^n |\pi(i) - \sigma(i)| \]
for $\pi, \sigma \in \mathfrak S_n$. It is well known~\citep{DiaGra77} that
\begin{equation} \label{eq:kt-l1} 
\inver (\pi, \sigma) \le \|\pi - \sigma \|_1 \le 2 \inver (\pi, \sigma) \,. 
\end{equation}
Hence the rates of estimation in the two distances coincide. Another distance on $\mathfrak S_n$ we use is the $\ell_\infty$-distance, defined as
$$
\| \pi - \sigma \|_\infty = \max_{i \in [n]} \, |\pi(i) - \sigma(i)| \,.
$$

Note that unlike existing literature on ranking from pairwise comparisons where metrics on the probability parameters are studied, we employ here distances that measure how far an item is from its true ranking.

\section{Main results} \label{sec:result}

In this section, we state our main results. Specifically, we establish the minimax rates of estimating $\pi^*$ in the Kendall tau distance (and thus in $\ell_1$ distance) for  noisy sorting under both sampling models \ref{model:1} and \ref{model:2}. The minimax estimator that we propose is intractable in general and we complement our results with an efficient estimator of $\pi^*$ which achieves near-optimal rates in both the Kendall tau and the $\ell_\infty$-distance, under the sampling model \ref{model:2}.

\subsection{Minimax rates of noisy sorting} \label{sec:minimax}

%is usually employed to achieve optimal upper bounds for a statistical model. Given the data matrix $A$ consisting of outcomes of pairwise comparisons, 

Under the noisy sorting model with latent permutation $\pi^* \in \mathfrak S_n$ and matrix of probabilities $M \in \Mns$, we determine the minimax rate of estimating $\pi^*$ in the following theorem. Let $\E_{\pi^*,M}$ denote the expectation with respect to the probability distribution of the observations in the noisy sorting model with underlying permutation $\pi^* \in \mathfrak{S}_n$ and matrix of probabilities $M\in \Mns$, in either sampling model.

\begin{theorem} \label{thm:minimax}
Fix $\lambda \in (0, \frac 12-c]$ where $c$ is a universal positive constant.  It holds that
\renewcommand*{\arraystretch}{1.5}
\[ \min_{\tilde \pi} \max_{\substack{\pi^* \in \mathfrak S_n\\ M \in \Mns}} \E_{\pi^*,M} [ \inver (\tilde \pi, \pi^*) ] \asymp \left\{
\begin{array}{ll}
\DS\frac{n}{p \lambda^{2}} \land n^2\,, & \text{in sampling model \ref{model:1}} \,, \\
\DS\frac{n^3}{N \lambda^{2}} \land n^2\,, & \text{in sampling model \ref{model:2}} \,,
\end{array}\right.
 \]
where the minimum is taken minimized over all permutation estimators $\tilde \pi \in \mathfrak S_n$ that are measurable with respect to the observations.
\end{theorem}

The theorem establishes the minimax rates for noisy sorting, including the case of partial observations and weak signals. The upper bounds in fact hold with high probability as shown in Theorem~\ref{thm:upper-inv-net}.
If the expected numbers of observations in the two sampling models \ref{model:1} and \ref{model:2} are of the same order, i.e., $p \binom{n}{2} \asymp N$, then the two rates coincide. In this sense, the two sampling models are statistically equivalent.
In sampling model \ref{model:1}, if $p=1$ and $\lambda$ is larger than a constant, then the rate of order $n$ recovers the upper bound proved by \cite{BraMos08}.

Note in particular the absence of logarithmic factor in the rates. Naively bounding the metric entropy of $\mathfrak S_n$ by $\log |\mathfrak S_n| \simeq n \log n$ actually yields a superfluous logarithmic term in the upper bound. To avoid it, we study the doubling dimension of $\mathfrak S_n$; see the discussion after Proposition~\ref{prop:ball-cover-pack}. Closing this logarithmic gap for other problems involving latent permutations~\citep{ColDal16, FlaMaoRig16,ShaBalGunWai17,PanWaiCou17} remains an open question.

The technical assumption $\lambda \le 1/2- c$ in Theorem~\ref{thm:minimax} is very mild, because we are interested in the ``noisy'' sorting model (meaning that the pairwise comparisons are noisy, or equivalently that $\lambda$ is not close to $\frac 12$). In fact the requirement that $\lambda$ be bounded away from $\frac 12$ can be lifted, in which case we establish upper and lower bounds that match up to a logarithmic factor of order $\log (1/\Delta)$, where $\Delta=1/2-\lambda$ (see Section~\ref{sec:proof}).

Finally, we note that the proof of Theorem~\ref{thm:minimax} holds even in the so-called \emph{semi-random} setting~\citep{BluSpe95,MakMakVij13}, in which observations are generated by one of the random procedures described above, but a ``helpful'' adversary is allowed to reverse the outcome of any comparison in which a weaker item beat a stronger item. Though these reversals appear benign at first glance, the presence of such an adversary can in fact worsen statistical rates of estimation in more brittle models such as stochastic block models and the related broadcast tree model~\citep{MoiPerWei16}. Our results indicate that no such degradation occurs for the rates of estimation in the noisy sorting problem.

\subsection{Efficient multistage sorting} \label{sec:ms}

The minimax upper bound in Theorem~\ref{thm:minimax} is established using a computationally prohibitive estimator, %prohibitively hard,
so we now introduce an efficient estimator of the underlying permutation that can be computed in time  $\tilde O(n^2)$.
In this section, we prove theoretical guarantees for this estimator under the noisy sorting model with probability matrix $M = \Mnss$ and observations sampled with replacement according to \ref{model:2} when $\lambda$ is bounded away from zero by a universal constant.
%In this section, we restrict our attention to the noisy sorting model with probability matrix $M=\Mnss$ and observations sampled with replacement according to \ref{model:2} for theoretical guarantees.
%, and leave the general case $M \in \Mns$ and the sampling without replacement model to future work.
No polynomial-time algorithm was  previously known to achieve near-optimal rates even in this simplified setting when $o(n^2)$ pairwise comparisons are observed.

Since we aim to prove guarantees up to constants, we may assume that we have $2N$ pairwise comparisons, and split them into two independent samples, each containing $N$ pairwise comparisons. The first sample is used to estimate the parameter $\lambda$ and the second one is used to estimate the permutation $\pi^*$.

First, we introduce a fairly simple estimator $\hat \lambda$ of $\lambda$ that can be described informally as follows: first sort in increasing order the items according to the number of wins. Then for any pair $(i,j)$ for which item $i$ is ranked $n/2$ positions higher than item $j$, it is very likely that item $i$ is stronger than item $j$ so that it beats item $j$ with probability $\frac 12 + \lambda$. We then average the $\Ber(\frac 12+\lambda)$ variables over all such pairs to obtain an estimator $\hat \lambda$ of $\lambda$.
More formally, we further split the first sample into two subsamples, each containing $N/2$ pairwise comparisons. Denote by $A_{i,j}'$ and $A_{i,j}''$ the number of wins item $i$ has against item $j$ in the first and second subsample, respectively.
%, among the two subsamples of comparisons respectively.
The estimator $\hat \lambda$ is given by the following procedure:

\begin{enumerate}
\item For each $i \in [n]$, associate with item $i$ a score $S_i = \sum_{j =1}^n A_{i,j}'$. 

\item Construct a permutation $\tilde \pi$ by sorting the scores $S_i$ in increasing order, i.e., $\tilde \pi$ is chosen so that $\tilde \pi(i) < \tilde \pi(j)$ if $S_{i} \le S_{j}$, with ties broken arbitrarily.

\item Define $\DS \hat \lambda = \frac{2}{N} \binom{n}{2} \binom{n/2}{2}^{-1} \sum_{\tilde \pi(i) - \tilde \pi (j) > \frac n2} A_{i,j}'' - \frac 12. $
\end{enumerate}

Given the estimator $\hat \lambda$, we now describe a multistage procedure to estimate the permutation $\pi^*$. To recover the underlying order of items, it is equivalent to estimate the row sums $\sum_{j=1}^n M_{\pi^*(i), \pi^*(j)}$ which we call scores of the items, because the scores are increasing linearly if the items are placed in order. Initially, for each $i \in [n]$, we estimate the score of item $i$ by the number of wins item $i$ has. If item $i$ has a much higher score than item $j$ in the first stage, then we are confident that item $i$ is stronger than item $j$. Hence in the second stage, we can estimate $M_{\pi^*(i), \pi^*(j)}$ by $\frac 12 + \hat \lambda$, which is very close to the truth. For those pairs that we are not certain about, $M_{\pi^*(i), \pi^*(j)}$ is still estimated by its empirical version. The variance of each score is thus greatly reduced in the second stage, thereby yielding a more accurate order of the items. Then we iterate this process to obtain finer and finer estimates of the scores and the underlying order.

To present the Multistage Sorting (\MS) algorithm formally, let us fix a positive integer $T$ which is the number of stages of the algorithm. We further split the second sample into $T$ subsamples each containing $N/T$ pairwise comparisons\footnote{We assume without loss of generality that $T$ divides $N$ to ease the notation.}. Similar to the data matrix $A$ for the full sample, for $t \in [T]$ we define a matrix $A^{(t)} \in \R^{n \times n}$ by setting $A^{(t)}_{i,j}$ to be the number of wins item $i$ has against item $j$ in the $t$-th sample. The \MS\ algorithm proceeds as follows:

\begin{enumerate}
\item For each $i \in [n]$, define $I^{(0)}(i) = [n]$, $I^{(0)}_-(i) = \varnothing$ and $I^{(0)}_+(i) = \varnothing$.

\item At the $t$-th stage where $t \in [T]$, compute the score $S^{(t)}_i$ of item $i$:
$$ S^{(t)}_i = \frac{Tn(n-1)}{2N} \sum_{j \in I^{(t-1)}(i) } A^{(t)}_{i,j} + \sum_{j \in I^{(t-1)}_-(i)} \big(\frac 12 + \hat \lambda \big) + \sum_{j \in I^{(t-1)}_+(i)} \big(\frac 12 - \hat \lambda \big) \,. $$

\item Let $C_0$ and $C_1$ be sufficiently large universal constants\footnote{Determined according to Lemma~\ref{lem:lambda} and Lemma~\ref{lem:subset-concentrate} respectively.}. If it holds that 
\begin{equation}
\label{EQ:Itlarge}
|I^{(t-1)}(i)| \ge C_1 n^2 \frac{T}{N} \log (nT)\,,
\end{equation}
then we  set the threshold 
$$ \tau_i^{(t)} = (10+2C_0) n \sqrt{|I^{(t-1)}(i)| T N^{-1} \log(nT) } \,,$$
and define the  sets
\begin{align*}
I^{(t)}_-(i) &= \{j \in [n]: S^{(t)}_j - S^{(t)}_i < - \tau_i^{(t)}\}, \\
I^{(t)}_+(i) &= \{j \in [n]: S^{(t)}_j - S^{(t)}_i > \tau_i^{(t)}\}, \text{ and} \\
I^{(t)}(i) &= [n] \setminus \big( I^{(t)}_-(i) \cup I^{(t)}_+(i) \big) \,.
\end{align*} 
If~\eqref{EQ:Itlarge} does not hold, then we define $I^{(t)}(i) = I^{(t-1)}(i)$, $I_-^{(t)}(i) = I_-^{(t-1)}(i)$ and $I_+^{(t)}(i) = I_+^{(t-1)}(i)$. Note that $I^{(t)}(i)$ denotes the set of items $j$ whose ranking relative to $i$ has not been determined by the algorithm at stage $t$.

\item After repeating Step 2 and 3 for $t = 1, \dots, T$, output a permutation $\hat \pi^{\scriptscriptstyle \MS}$ by sorting the scores $S^{(T)}_i$ in increasing order, i.e., $\hat \pi^{\scriptscriptstyle \MS}$ is chosen so that $\hat \pi^{\scriptscriptstyle \MS} (i) < \hat \pi^{\scriptscriptstyle \MS} (j)$ if $S^{(T)}_{i} \le S^{(T)}_{j}$ with ties broken arbitrarily.
\end{enumerate}

It is clear that the time complexity of each stage of the algorithm is $O(n^2)$. Take $T = \lfloor \log \log n \rfloor$ so that the overall time complexity of the \MS\ algorithm is only $O(n^2 \log \log n)$.
Our main result in this section is the following guarantee on the performance of the estimator $\hat \pi^{\scriptscriptstyle \MS}$ given by the \MS\ algorithm.

\begin{theorem} \label{thm:multistage}
Suppose that $N \ge C n \log n$ for a sufficiently large constant $C>0$ and that $M = \Mnss$ where $\lambda \in [c,\frac 12)$ for a constant $c>0$. Then, under the noisy sorting model with sampling model~\ref{model:2}, the following holds. With probability at least $1-n^{-7}$, the \MS\ algorithm with $T=\lfloor \log \log n \rfloor$  stages outputs an estimator $\hat \pi^{\scriptscriptstyle \MS}$ that satisfies
$$ \|\hat \pi^{\scriptscriptstyle \MS} - \pi^*\|_\infty \lesssim \frac{n^2}{N} (\log n)\log \log n$$ 
and
$$
\inver(\hat \pi^{\scriptscriptstyle \MS}, \pi^*) \lesssim \frac{n^3}{N} (\log n)\log \log n \,.
$$
\end{theorem}

Note that the second statement follows from the first one together with~\eqref{eq:kt-l1}. Indeed, we have
$$
\inver(\hat \pi^{\scriptscriptstyle \MS}, \pi^*) \le \|\hat \pi^{\scriptscriptstyle \MS} - \pi^*\|_1 \le n \|\hat \pi^{\scriptscriptstyle \MS} - \pi^*\|_\infty \lesssim \frac{n^3}{N} (\log n) \log \log n \,,
$$
which is optimal up to a polylogarithmic factor in the regime where $\lambda$ is bounded away from $0$ according to Theorem~\ref{thm:minimax} (and Theorem~\ref{thm:lower-inv}). Therefore, the \MS\ algorithm achieves significant computational efficiency while sacrificing little in terms of statistical performance. On the downside, it is limited to the noisy sorting model where $M=M_n^*(\lambda)$---this assumption is necessary to exploit strong stochastic transitivity---and our analysis does not account for the dependence in $\lambda$.

Furthermore, although we only consider model \ref{model:2} of sampling with replacement in this section, the \MS\ algorithm can be easily modified to handle model \ref{model:1} of sampling without replacement. It is much more challenging to prove analogous theoretical guarantees in this case, because we cannot split the observations into independent samples. In Section~\ref{sec:sim}, however, we provide empirical evidence showing that the \MS\ estimator has very similar performance for the two sampling models.

Our algorithm bears comparison with the algorithm proposed by \cite{BraMos08}. Their algorithm---which works in the full observation case $N = \binom n2$---achieves the statistically optimal rate in time $O(n^C)$, where $C$ is a large positive constant depending on $\lambda$. Though our algorithm's statistical performance falls short of the optimal rate by a polylogarithmic factor, it runs in time $O(n^2 \log \log n)$ and works in the partial observation setting as long as $N \gtrsim n \log n$. Note by way of comparison that Theorem~\ref{thm:lower-inv} indicates that no procedure achieves nontrivial recovery unless $N \gg n$.

\section{Simulations} \label{sec:sim}

To support our theoretical findings in Section~\ref{sec:ms}, we implement the \MS\ algorithm on synthetic instances generated from the noisy sorting model. For simplicity, we take $\lambda = 0.25$ and set $\hat \lambda = \lambda$ in the algorithm. Theorem~\ref{thm:multistage} predicts a scaling $n^3 N^{-1} (\log n) \log \log n$ of the estimation error in the Kendall tau distance for model \ref{model:2} of sampling with replacement, where $n$ is the number of items and $N$ is the number of pairwise comparisons. This rate is optimal up to a polylogarithmic factor according to Theorem~\ref{thm:lower-inv}.

%\ndpr{Use terminology "with/without replacement" in the legend. Also, use thicker lines with different colors. If you do not do all integer values on the x-axis, mark the points that you measure with a large circle (I guess this is MATLAB)}
\begin{figure}[ht]
\centering
\begin{minipage}[c]{.49\linewidth}
\includegraphics[clip, trim=1.4cm 6.3cm 1.8cm 6.7cm, width=\linewidth]{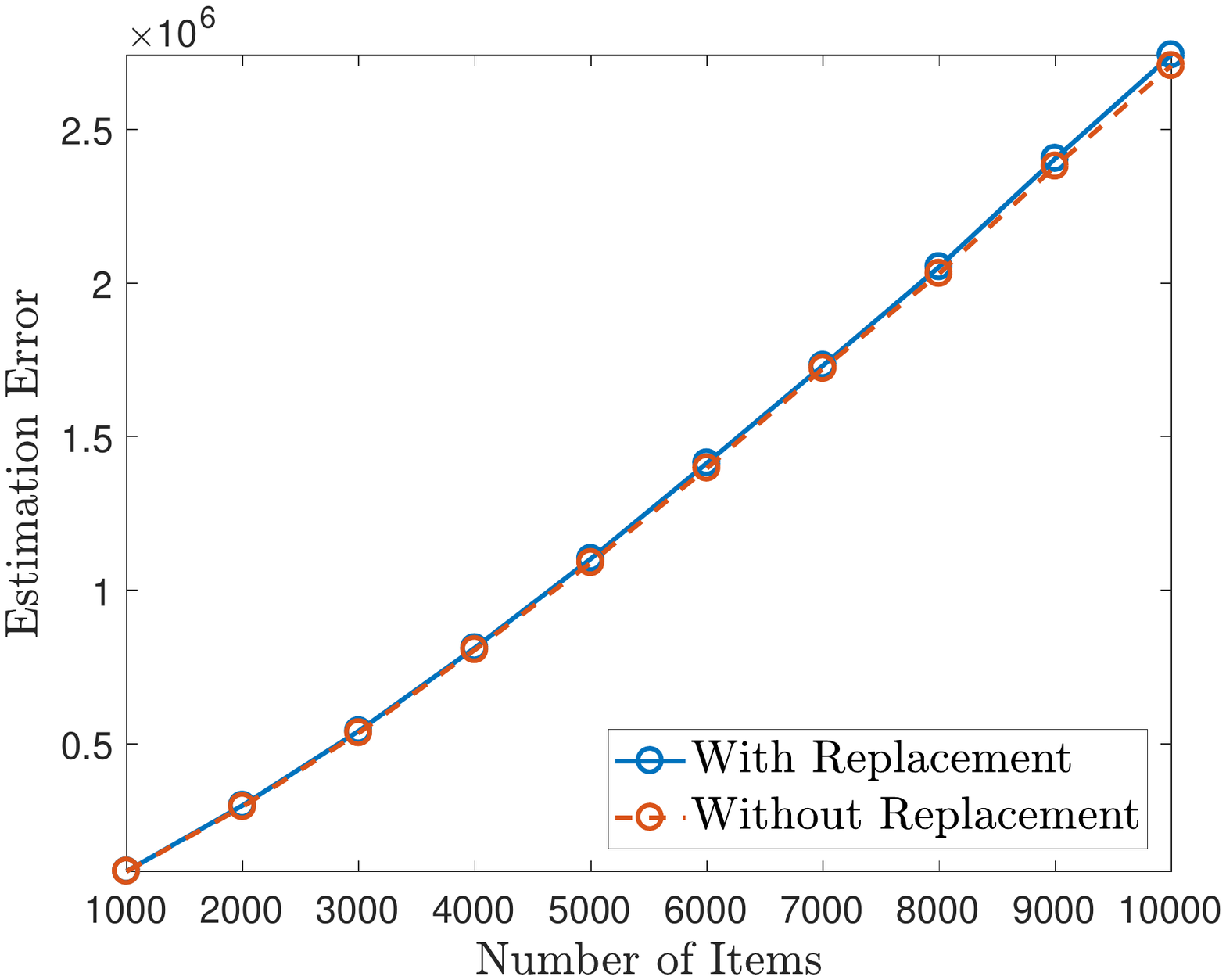}
\end{minipage}
\begin{minipage}[c]{.49\linewidth}
\includegraphics[clip, trim=1.4cm 6.3cm 1.8cm 6.7cm, width=\linewidth]{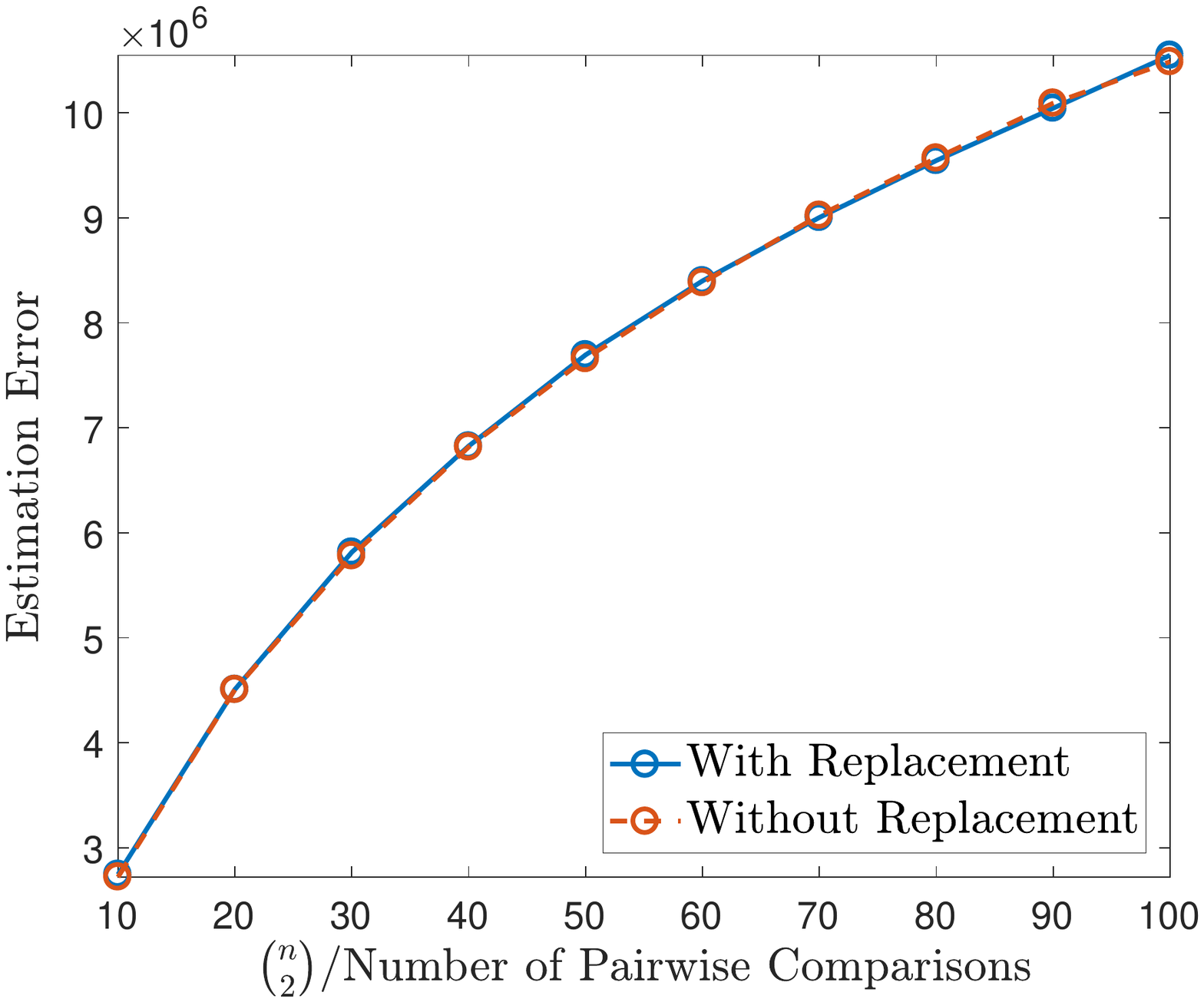}
\end{minipage} 
\caption{Estimation errors $\inver(\hat \pi^{\scriptscriptstyle \MS}, \pi^*)$ for the observations sampled with and without replacement. %sampling models of independent sampling and missing comparisons.
Left: $N=p\binom{n}{2}=0.1\binom{n}{2}$ and $n$ ranging from $1,000$ to $10,000$; Right: $n=10,000$ and $N=p\binom{n}{2}$ ranging from $0.1\binom{n}{2}$ to $0.01\binom{n}{2}$.}
\label{fig:np}
\end{figure}

In Figure~\ref{fig:np}, we plot estimation errors $\inver(\hat \pi^{\scriptscriptstyle \MS}, \pi^*)$  averaged over $10$ instances generated from the model.
In the left plot, we let $n$ range from $1,000$ to $10,000$ and set $N=0.1\binom{n}{2}$. For this choice of $N$, Theorem~\ref{thm:multistage} predicts that $\inver(\hat \pi^{\scriptscriptstyle \MS}, \pi^*) = \tilde O_{\p}(n)$ and we indeed observe a near-linear scaling in that plot. In the right plot, we fix $n=10,000$ and let the proportion of observed entries, $\alpha=N/\binom{n}{2}$ range from $.01$ to $.1$. For this choice of parameters, Theorem~\ref{thm:multistage} predicts that $\inver(\hat \pi^{\scriptscriptstyle \MS}, \pi^*) \le C_n\alpha^{-1}$ (recall that here $n$ is fixed), and we clearly observe a sublinear relation between $\inver(\hat \pi^{\scriptscriptstyle \MS}, \pi^*)$ and $\alpha^{-1}$.
%\ndpr{what exponent does a log-log plot predict?}
Note that this does not contradict the lower bound since the latter is stated up to constants.

Moreover, the \MS\ algorithm can be easily modified to work for the without replacement model \ref{model:1}. Namely, given the partially observed pairwise comparisons, we assign each comparison to one of the samples $1, \dots, T$ uniformly at random, independent of all the other assignments. After splitting the whole sample into $T$ subsamples, we execute the \MS\ algorithm as in the previous case. In Figure~\ref{fig:np}, we take $p = N/\binom n2$ and plot the estimation errors for sampling without replacement, which closely follow the errors for observations sampled with replacement. Therefore, although it seems difficult to prove analogous guarantees on the performance of the \MS\ algorithm applied to the without replacement model, empirically the algorithm performs very similarly for the two sampling models.

\begin{figure}[ht]
\centering
\begin{minipage}[c]{.32\linewidth}
\includegraphics[clip, trim=4.2cm 9.1cm 4.3cm 9.2cm, width=\linewidth]{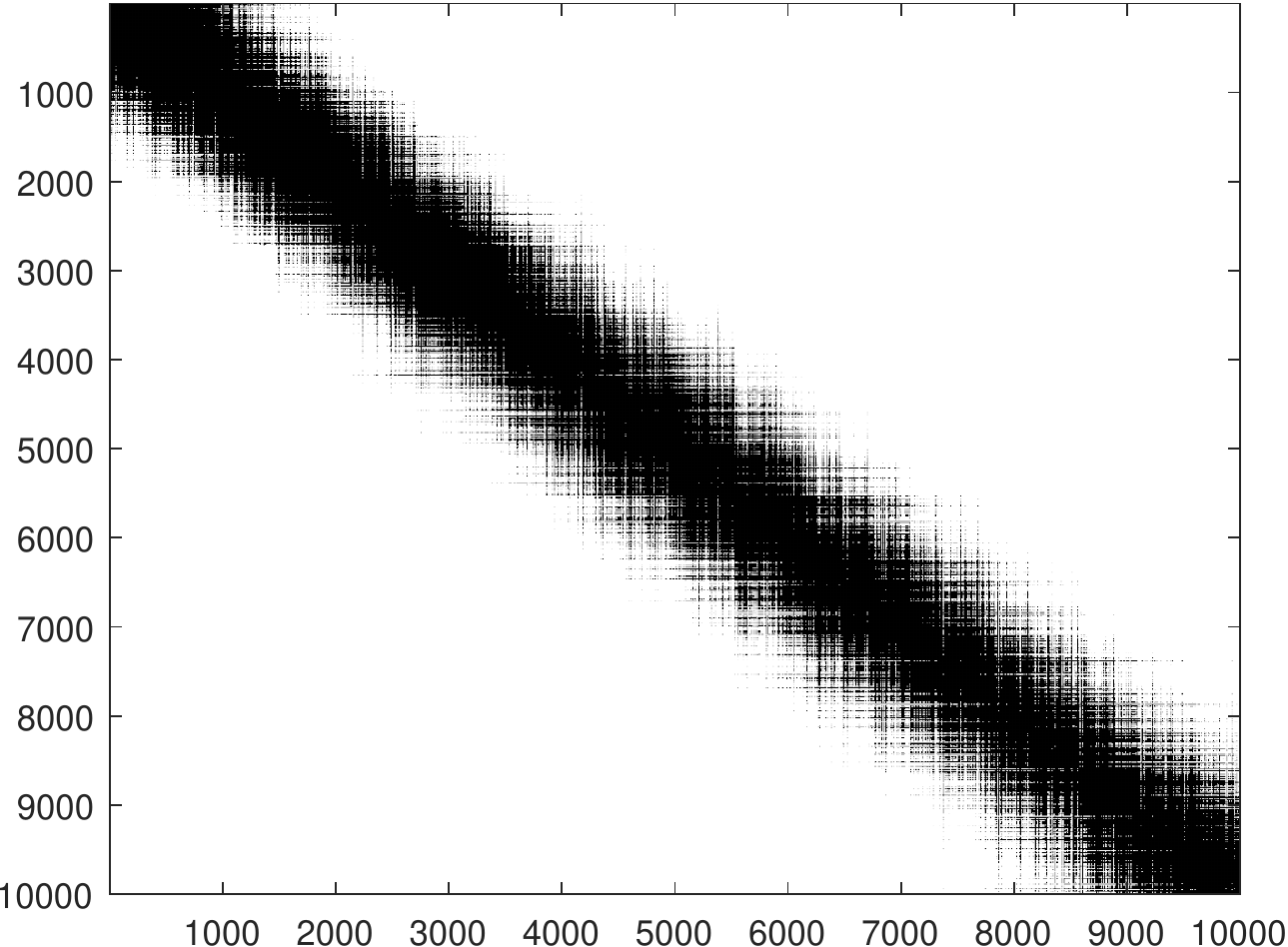}
\centerline{\scriptsize \sf Stage 1}
\end{minipage}
\begin{minipage}[c]{.32\linewidth}
\includegraphics[clip, trim=4.2cm 9.1cm 4.3cm 9.2cm, width=\linewidth]{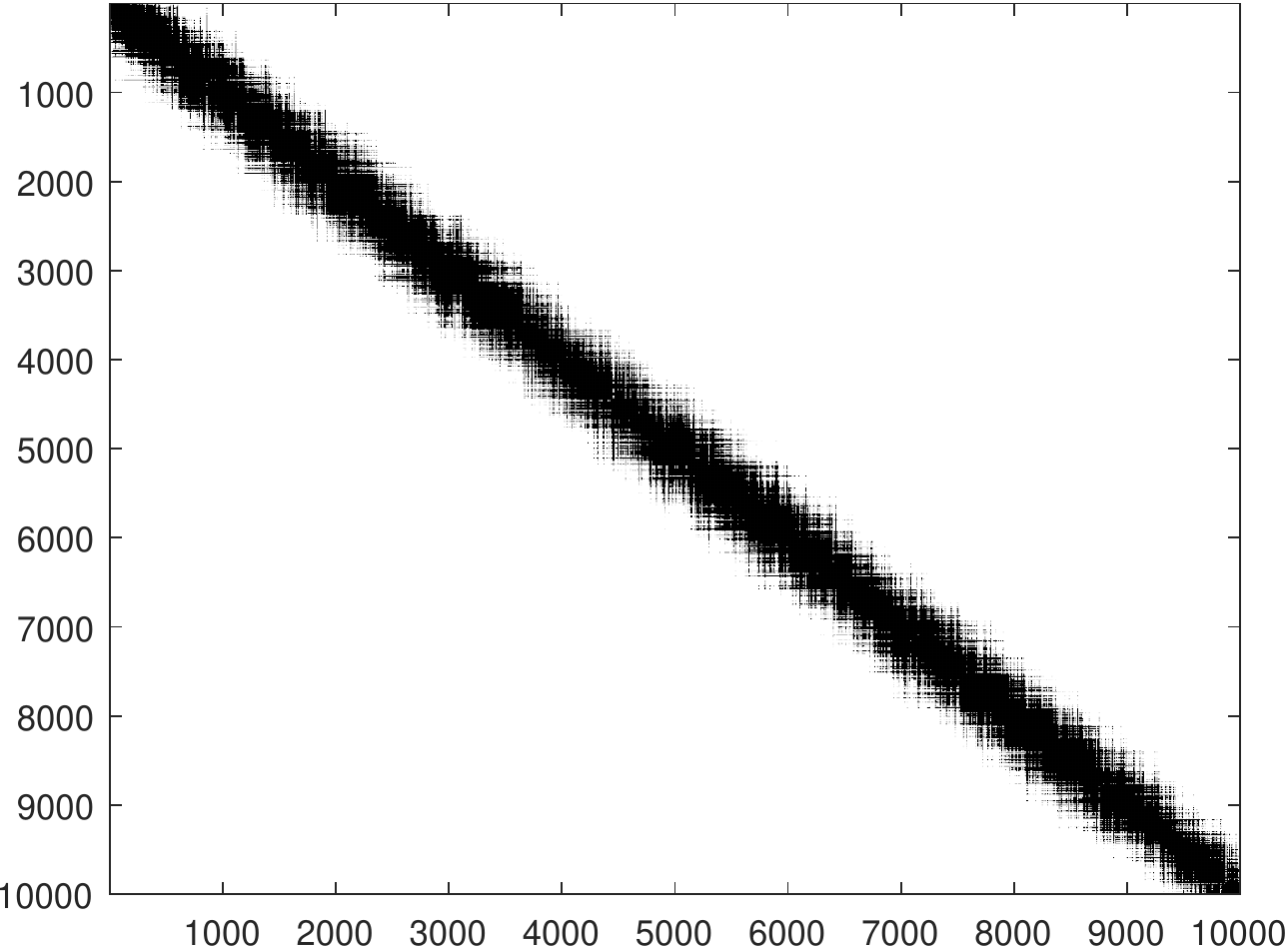}
\centerline{\scriptsize \sf Stage 2}
\end{minipage} 
\begin{minipage}[c]{.32\linewidth}
\includegraphics[clip, trim=4.2cm 9.1cm 4.3cm 9.2cm, width=\linewidth]{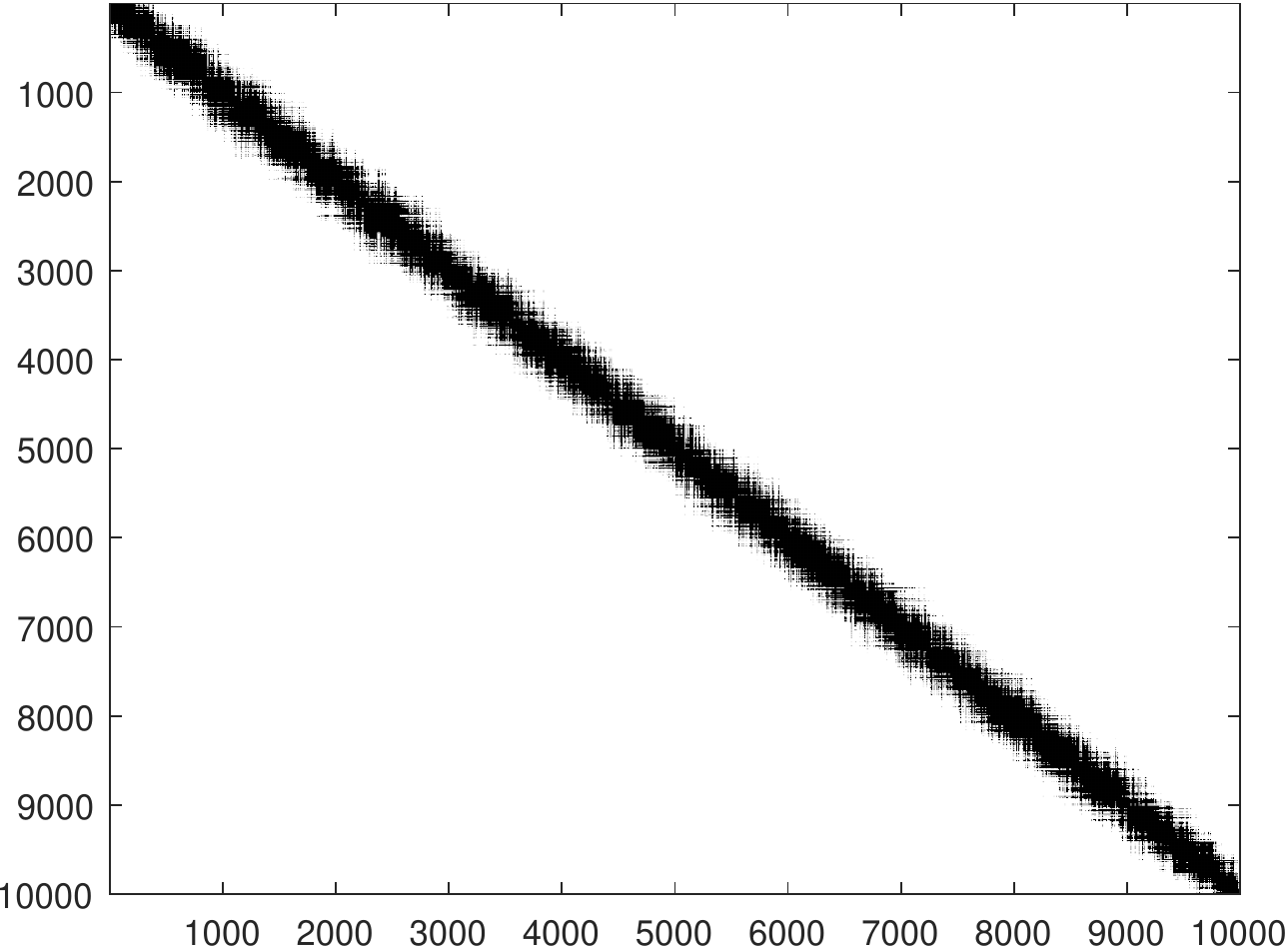}
\centerline{\scriptsize \sf Stage 3}
\end{minipage} 
\caption{The uncertainty regions $\mathcal{R}^{(t)}$ at stages $t=1,2,3$ of the \MS\ algorithm. The two axes represent the indices of the items. A black pixel at $(i,j)$ indicates that $(i,j) \in \mathcal{R}^{(t)}$, i.e., the algorithm is not certain about the relative order of item $i$ and item $j$ at stage $t$. A white pixel indicates the opposite.}
\label{fig:steps}
\end{figure}

To gain further intuition about the \MS\ algorithm, we consider the set $I^{(t)}(i)$ defined in the algorithm. At stage $t$ of the algorithm, the set $I^{(t)}(i)$ consists of all indices $j$ for which we are not certain about the relative order of item $i$ and item $j$. The proof of Theorem~\ref{thm:multistage} essentially shows that the uncertainty set $I^{(t)}(i)$ is shrinking as the algorithm proceeds. To verify this intuition, in Figure~\ref{fig:steps} we plot the \emph{uncertainty regions}  $$\mathcal{R}^{(t)} \defn \big\{ (i,j) \in [n]^2: i \in [n],\, j \in I^{(t)}(i) \big\}$$ 
at stages $t = 1,2,3$ of the \MS\ algorithm, for $n=10,000$ and $N=\binom{n}{2}$. The items are ordered according to $\pi^* = \id$ for visibility of the region. As exhibited in the plots, the uncertainty region is indeed shrinking as the algorithm proceeds.

\section{The symmetric group and inversions} \label{sec:inver}

Before proving the main results for the noisy sorting model, we study the metric entropy of the symmetric group $\mathfrak S_n$ with respect to the Kendall tau distance. Counting permutations subject to constraints in terms of the Kendall tau distance is of theoretical importance and has interesting applications, e.g., in coding theory \citep[see, e.g,][]{BarMaz10,MazBarZem13}. We present the results in terms of metric entropy, which easily applies to the noisy sorting problem and may find further applications in statistical problems involving permutations.

For $\eps > 0$ and $S \subseteq \mathfrak S_n$, let $N(S, \eps)$ and $D(S, \eps)$ denote respectively the $\eps$-covering number and the $\eps$-packing number of $S$ with respect to the Kendall tau distance.
%All nets and packings we discuss will be with respect to the Kendall tau distance unless otherwise specified.
The following main result of this section provides bounds on the metric entropy of balls in $\mathfrak S_n$.

\begin{proposition} \label{prop:ball-cover-pack}
Consider the ball $\cB(\pi, r) = \{\sigma \in \mathfrak S_n: \inver(\pi, \sigma) \le r\}$
centered at $\pi \in \mathfrak S_n$ with radius $r \in (0, \binom n2]$. 
We have that for $\eps \in (0,r)$,
\begin{align*}
n \log \big( \frac{r}{n+\eps} \big) - 2n &\le \log N (\cB(\pi, r), \eps ) \le \log D (\cB(\pi, r),  \eps ) \le  n \log \big( \frac{ 2n + 2r}{\eps} \big) + 2n \,.
\end{align*}
\end{proposition}

We now discuss some high-level implications of Proposition~\ref{prop:ball-cover-pack}. Note that if $n \lesssim \eps < r \le \binom n2$, the lemma states that the $\eps$-metric entropy of a ball of radius $r$ in the Kendall tau distance scales as $n \log \frac{r}{\eps}$. In other words, the symmetric group $\mathfrak S_n$ equipped with the Kendall tau metric is a doubling space with doubling dimension $\Theta(n)$. 
%This could be viewed as an explanation of nice properties the symmetric group exhibits when equipped with $\inver$. Particularly, if we consider the noisy sorting model with full observations (i.e., $p=1$ or $N \asymp n^2$) and $\lambda = \frac 14$ for simplicity, then the normalized minimax rate given by Theorem~\ref{thm:upper-inv-net} and Theorem~\ref{thm:lower-inv} is $n^{-2} \inver(\hat \pi, \pi^*) \asymp n^{-1}$. At a high-level, this is consistent with the fact that the intrinsic doubling dimension of $\mathfrak S_n$ is $\Theta(n)$ and we have $\Theta(n^2)$ observations. 
One of the main messages of the current work is that although $\log |\mathfrak S_n| = \log (n!) \asymp n \log n$,  the intrinsic dimension of $\mathfrak S_n$ is $\Theta(n)$, which explains the absence of logarithmic factor in the minimax rate.

To start the proof, we first recall a useful tool for counting permutations, the \emph{inversion table}. Formally, the inversion table $b_1, \dots, b_n$ of a permutation $\pi \in \mathfrak S_n$ is defined by
\[ b_i = \sum_{j: i<j} \1 \big( \pi(i) > \pi(j) \big) \]
for $i \in [n]$. Clearly, we have that $b_i \in \{0,1, \dots, n-i\}$ and $\inver(\pi, \id) = \sum_{i=1}^n b_i$. It is easy to reconstruct a unique permutation using an inversion table with $b_i \in \{0,1, \dots, n-i\}, i \in [n]$, so the set of inversion tables is bijective to $\mathfrak S_n$ via this relation; see, e.g., \cite{Mah00}. 
%More interestingly, if we sample $b_i$ independently from a uniform distribution on $\{0,1, \dots, n-i\}$ for each $i \in [n]$, then the bijection yields a random permutation uniformly distributed over $\mathfrak S_n$, and vice versa \cite{Mah00}. 
We use this bijection to bound the number of permutations that differ from the identity by at most $k$ inversions. The following lemma appears in a different form in~\citet{BarMaz10}. We provide a simple proof here for completeness.

\begin{lemma} \label{lem:inversions}
For $0 \le k \le \binom{n}{2}$, we have that
%\[ \lfloor k/n \rfloor ! \, (k/n)^{n-k/n} \le \big| \{ \pi \in  \mathfrak S_n: \inver(\pi) \le k \} \big| \le e^n (1+k/n)^n \,. \]
\[ n \log (k/n) - n \le \log \big| \{ \pi \in  \mathfrak S_n: \inver(\pi, \id) \le k \} \big| \le n \log (1+k/n) + n \,. \]
\end{lemma}

\begin{proof}
According to the discussion above, the cardinality $\big| \{ \pi \in  \mathfrak S_n: \inver(\pi, \id) \le k \} \big|$, which we denote by $L$, is equal to the number of inversion tables $b_1, \dots, b_n$ where $b_i \in \{0,1, \dots, n-i\}$ such that $\sum_{i=1}^n b_i \le k$.
%According to the discussion above, the cardinality in interest, which we denote by $L$, is the number of inversion tables $b_1, \dots, b_n$ where $b_i \in \{0,1, \dots, n-i\}$ such that $\sum_{i=1}^n b_i \le k$.
On the one hand, 
if $b_i \le \lfloor k/n \rfloor$ for all $i \in [n]$, then $\sum_{i=1}^n b_i \le k$, so a lower bound on $L$ is given by
\begin{align*} 
L &\ge\prod_{i=1}^{n} (\lfloor k/n \rfloor + 1)\land (n-i+1) \\
&\ge \prod_{i=1}^{n-\lfloor k/n \rfloor} (\lfloor k/n \rfloor + 1) \prod_{i=n-\lfloor k/n \rfloor +1}^n (n-i+1) \\
&\ge (k/n)^{n-k/n} \lfloor k/n \rfloor ! \,. 
\end{align*}
Using Stirling's approximation, we see that
\begin{align*} 
\log L &\ge n \log (k/n) - (k/n) \log (k/n) + \lfloor k/n \rfloor \log \lfloor k/n \rfloor - \lfloor k/n \rfloor \\
&\ge n \log (k/n) - n \,. 
\end{align*}

On the other hand, if $b_i$ is only required to be a nonnegative integer for each $i \in [n]$, then we can use a standard ``stars and bars" counting argument \citep{Fel68} to get an upper bound of the form
\[ L \le \binom{n+k}{n} \le 
e^n (1+k/n)^n \,. \]
Taking the logarithm finishes the proof.
% JW better bounds for k < n, will this matter?
% LB: \binom{n-1}{k} \geq \binom{n}{k}/n => log >= k log n/k - log n
% UB: k log(1 + n/k) + k
\end{proof}

%\begin{lemma} \label{lem:ball-cover-pack}
%Consider the ball $\cB(\pi, r) = \{\sigma \in \mathfrak S_n: \inver(\pi, \sigma) \le r\}$
%centered at $\pi \in \mathfrak S_n$ with radius $r \in (0, \binom n2]$. 
%We have that for $\eps \in (0,r)$,
%\begin{align*}
%n \log \big( \frac{r}{2n+2\eps} \big) - 3n &\le \log N (\cB(\pi, r), \inver, \eps ) \\
%&\le \log D (\cB(\pi, r), \inver, \eps ) \le  n \log \big( \frac{ 2n + 6r}{\eps} \big) + 3n \,.
%\end{align*}
%\end{lemma}

We are ready to prove Proposition~\ref{prop:ball-cover-pack}.

\begin{proof}[of Proposition~\ref{prop:ball-cover-pack}]
The relation between the covering and the packing number is standard.

We employ a standard volume argument to control these numbers.
Let $\cP$ be a $2\eps$-packing of $\cB(\pi, r)$ so that the balls $\cB(\sigma, \eps)$ are disjoint for $\sigma \in \cP$.
Moreover, by the triangle inequality, $\cB(\sigma, \eps) \subseteq \cB(\pi, r + \eps)$ for each $\sigma \in \cP$.
By the invariance of the Kendall tau distance under composition, Lemma~\ref{lem:inversions} yields
\begin{align*}
\log D(\cB(\pi, r), 2\eps )& \leq n \log(1+r/n) + n - n \log(\eps/n) + n \\
& = n \log\big(\frac{n+r}{\eps}\big) + 2n\,.
\end{align*}

On the other hand, if $\cN$ is an $\eps$-net of $\cB(\pi, r)$, then the set of balls $\{\cB(\sigma, \eps)\}_{\sigma \in \cN}$ covers $\cB(\pi, r)$.
By Lemma~\ref{lem:inversions}, we obtain
\begin{align*}
\log N(\cB(\pi, r), \eps) & \geq \log |\cB(\pi, r)| - \log |\cB(\sigma, \eps)| \\
& \geq n \log(r/n) - n - n \log(1+\eps/n) - n \\
& = n \log \big(\frac{r}{n+\eps}\big) - 2n\,,
\end{align*}
as claimed.
\end{proof}

The lower bound on the packing number in Proposition~\ref{prop:ball-cover-pack} becomes vacuous when $r$ and $\eps$ are smaller than $n$, so we complement it with the following result, which is useful for proving minimax lower bounds.

\begin{lemma} \label{lem:pack-special}
Consider the ball $\cB(\pi, r)$
where $r <n/2$. 
We have that
\[
\log N (\cB(\pi, r),  r/4 ) \ge \frac r5 \log \frac nr \,.
\]
\end{lemma}

\begin{proof}
Without loss of generality, we may assume that $\pi = \id$ and $n$ is even.
The sparse Varshamov-Gilbert bound \citep[][Lemma~4.10]{Mas07} states that there exists a set $\cS$ of $r$-sparse vectors in $\{0,1\}^{n/2}$, such that $\log |\cS| \ge \frac r5 \log \frac n{r}$ and any two distinct vectors in $\cS$ are separated by at least $r/2$ in the Hamming distance.
We now map every $v \in \cS$ to a permutation $\pi \in \cB(\id, r)$ by defining
\begin{enumerate}
\item $\pi(2i-1) = 2i-1$ and $\pi(2i) = 2i$ if $v(i) = 0$, and
\item $\pi(2i-1) = 2i$ and $\pi(2i) = 2i-1$ if $v(i) = 1$,
\end{enumerate}
for $i \in [n]$. Note that $\pi \in \cB(\id,r)$ because $\pi$ swaps at most $r$ adjacent pairs. Denote by $\cP$ the image of $\cS$ under this mapping. Since the Hamming distance between any two distinct vectors in $\cS$ is lower bounded by $r/2$, we see that $\inver(\pi, \sigma) \ge r/2$ for any distinct $\pi, \sigma \in \cP$. Thus $\cP$ is an $r/2$-packing of $\cB(\id,r)$. 
By construction, $|\cP| = |\cS| \ge \frac r5 \log \frac nr$, so we can use the standard relation $D (\cB(\id, r),  r/2 )  \le N (\cB(\id, r),  r/4 )$ to complete the proof.
\end{proof}

\section{Proofs of the main results} \label{sec:proof}

This section is devoted to the proofs of our main results. We start with a lemma giving useful tail bounds for the binomial distribution.

\begin{lemma} \label{lem:bin-tail}
Suppose that $X$ has the Binomial distribution $\Bin(N, p)$ where $N \in \Z_+$ and $p \in (0,1)$. Then for $r \in (0,p)$ and $s \in (p,1)$, we have
\begin{enumerate}
\item
$\p (X \le rN) \le \exp \big( -N \frac{(p-r)^2}{2 p (1-r)} \big) ,$ and
\item
$\p (X \ge sN) \le \exp \big( -N \frac{(p-s)^2}{2 s (1-p)} \big) .$
\end{enumerate}
\end{lemma}

\begin{proof}
First, for $0<q<p<1$, by the definition of the Kullback-Leibler divergence, we have
\begin{align}
\KL \big(\Ber(p) \| \Ber(q) \big) &= p \log \frac pq + (1-p) \log \frac{1-p}{1-q}
= \int_q^p \big(\frac px - \frac{1-p}{1-x}\big) \, d x \nonumber \\
& = \int_q^p \frac{p-x}{x(1-x)} \,d x 
\ge \int_q^p \frac{p-x}{p (1-q)} \,d x 
= \frac{(p-q)^2}{2 p  (1-q)} \,. \label{eq:ber-kl-1}
\end{align}
%where the inequality holds because $x(1-x) \le x \land (1-x) \le p \land (1-q)$.
Thus we also have
\begin{equation} \label{eq:ber-kl-2} 
\KL \big(\Ber(q) \| \Ber(p) \big) = \KL \big(\Ber(1-q) \| \Ber(1-p) \big)
\ge \frac{(p-q)^2}{2 p  (1-q)} \,. 
\end{equation}
Moreover, by Theorem~1 of \cite{ArrGor89} and symmetry, it holds that
\begin{enumerate}
\item
$\p (X \le rN) \le \exp( -N \KL(\Ber(r) \| \Ber(p)) ) ,$ and 
\item
$\p (X \ge sN) \le \exp (- N \KL(\Ber(s) \| \Ber(p)) ) .$
\end{enumerate}
The claimed tail bounds hence follow from \eqref{eq:ber-kl-1} and \eqref{eq:ber-kl-2}.
\end{proof}

\subsection{Proof of Theorem~\ref{thm:minimax}}

First, to achieve optimal upper bounds, we consider a variant of maximum likelihood estimation. Fix $\lambda \in (0,1/2), p \in (0,1]$ and define $\varphi = n p^{-1} \lambda^{-2}$ in the case of sampling model \ref{model:1}, and  $\varphi = n^3 N^{-1} \lambda^{-2}$ in the case of sampling model \ref{model:2}. If $\lambda$ or $p$ is unknown, one may learn these scalar parameters easily from the observations and define $\varphi$ using the estimated values. For readability, we assume that they are given to avoid these technical complications.

Let $\cP$ be a maximal $\varphi$-packing (and thus a $\varphi$-net) of the symmetric group $\mathfrak S_n$ with respect to $\inver$.
Consider the following estimator:
%Then we define $\hat \pi$ to be the MLE over $\cP$, i.e.,
\begin{equation} \label{eq:mle-net} 
\hat \pi \in \operatorname*{argmax}_{\pi \in \cP}  \sum_{\pi(i) > \pi(j)} A_{i,j}  \,. 
\end{equation}
It is easy to see that $\hat \pi$ is the MLE of $\pi^*$ over $\cP$. Such an estimator is often called \emph{sieve estimator}~\citep[see, e.g.][]{LeC86} in the statistics literature. The estimator $\hat \pi$ satisfies the following upper bounds.

%Since maximizing the likelihood of the model is equivalent to maximizing the number of comparisons in which stronger items beat weaker items, it is easy to see that the MLE $\hat \pi$ over $\cP$ indeed takes the form \eqref{eq:mle-net}. 

%We assume $\lambda^* \in [c, \frac 12 -c]$ where $c$ is a universal positive constant, as this is the most interesting regime. Moreover, we consider $\frac 1n \lesssim p \le 1$ or correspondingly $n \lesssim N \lesssim n^2$, for which consistent recovery is possible.

\begin{theorem} \label{thm:upper-inv-net}
Consider the noisy sorting model with underlying permutation $\pi^*$ and probability matrix $M \in \Mns$ where $\lambda \in (0, \frac 12)$. Then, with probability at least $1-e^{-n/8}$, the estimator $\hat \pi$ defined in \eqref{eq:mle-net} satisfies 
\[ \inver (\hat \pi, \pi^*) \lesssim \left\{
\begin{array}{ll}
\DS\frac{n}{p \lambda^{2}} \land n^2 & \text{in model \ref{model:1}}\\
\DS\frac{n^3}{N \lambda^{2}} \land n^2 & \text{in model \ref{model:2}}\,.
\end{array}\right.
 \]
%\[ \inver (\hat \pi, \pi^*) \lesssim n p^{-1} (\lambda^*)^{-2} \land n^2 \]
%with probability at least $1-e^{-n/8}$ for the sampling model \ref{model:1}, and that
%\[ \inver (\hat \pi, \pi^*) \lesssim \frac{n^3}{N \lambda^{2}} \land n^2 \]
%%\[ \inver (\hat \pi, \pi^*) \lesssim n^3 N^{-1} (\lambda^*)^{-2} \land n^2 \]
%with probability at least $1-e^{-n/8}$ for the sampling model \ref{model:2}.
\end{theorem}

By integrating the tail probabilities of the above bounds, we easily obtain bounds on the expectation $\E[\inver (\hat \pi, \pi^*)]$ of the same order, which then prove the upper bounds in Theorem~\ref{thm:minimax}.
One may wonder whether the rate in Theorem~\ref{thm:upper-inv-net} can be achieved by the MLE $\check \pi$ over $\mathfrak S_n$ defined by
$$
\check \pi\in \operatorname*{argmax}_{\pi \in \mathfrak S_n}  \sum_{\pi(i) > \pi(j)} A_{i,j}\,.
$$
Our current techniques only allow us to prove bounds on $\inver (\check \pi, \pi^*)$ that incur an extra factor $\log(1/p\lambda)$ (resp. $\log(n^2/N\lambda)$) in model~\ref{model:1} (resp. \ref{model:2}). It is unclear whether these logarithmic factors can be removed for the MLE.

\medskip

\begin{proof}[of Theorem~\ref{thm:upper-inv-net}]
We assume that $n$ is lower bounded by a constant without loss of generality, and note that the bounds of order $n^2$ are trivial. The proof is split into four parts to improve readability.

\newcommand{\DD}{\mathfrak{D}}
\paragraph{Basic setup.} 
Since $\cP$ is a maximal $\varphi$-packing of $\mathfrak S_n$, it is also a $\varphi$-net and thus there exists $\tilde \pi \in \cP$ such that $\DD \defn \inver(\tilde \pi, \pi^*) \le \varphi$. 
By  definition of $\hat \pi$,
\( \sum_{\hat \pi(i) < \hat \pi(j)} A_{i,j} \le \sum_{\tilde \pi(i)<\tilde \pi(j)} A_{i, j} .  \)
Canceling  concordant pairs $(i,j)$ under $\hat \pi$ and $\tilde \pi$, we see that
\[
\sum_{\hat \pi(i) < \hat \pi(j), \, \tilde \pi(i)>\tilde \pi(j)} A_{i,j} \le \sum_{\hat \pi(i) > \hat \pi(j), \, \tilde \pi(i)<\tilde \pi(j)} A_{i,j} \,. 
\]
Splitting the summands according to $\pi^*$ yields that
\[
\sum_{\substack{\scriptscriptstyle \hat \pi(i) < \hat \pi(j) , \\ \scriptscriptstyle \tilde \pi(i)>\tilde \pi(j) , \\ \scriptscriptstyle \pi^*(i) < \pi^*(j)}} A_{i,j} + \sum_{\substack{\scriptscriptstyle \hat \pi(i) < \hat \pi(j) , \\ \scriptscriptstyle \tilde \pi(i)>\tilde \pi(j) , \\ \scriptscriptstyle \pi^*(i) > \pi^*(j)}} A_{i,j} \le \sum_{\substack{\scriptscriptstyle \hat \pi(i) > \hat \pi(j) , \\ \scriptscriptstyle \tilde \pi(i)<\tilde \pi(j), \\ \scriptscriptstyle \pi^*(i) < \pi^*(j)}} A_{i,j} + \sum_{\substack{\scriptscriptstyle \hat \pi(i) > \hat \pi(j) , \\ \scriptscriptstyle \tilde \pi(i)<\tilde \pi(j) , \\ \scriptscriptstyle \pi^*(i) > \pi^*(j)}} A_{i,j} \,. 
\]
Since $A_{i,j} \ge 0$, we may drop the leftmost term and drop the condition $\hat \pi(i) > \hat \pi(j)$ in the rightmost term to obtain that 
%\notecm{This step might be what leads to the suboptimal dependence on $a$ and $b$.}
\begin{equation} \label{eq:key-ineq}
\sum_{\substack{\scriptscriptstyle \hat \pi(i) < \hat \pi(j) , \\ \scriptscriptstyle \tilde \pi(i)>\tilde \pi(j) , \\ \scriptscriptstyle \pi^*(i) > \pi^*(j)}} A_{i,j} \le \sum_{\substack{\scriptscriptstyle \hat \pi(i) > \hat \pi(j) , \\ \scriptscriptstyle \tilde \pi(i)<\tilde \pi(j) , \\ \scriptscriptstyle \pi^*(i) < \pi^*(j)}} A_{i,j} + \sum_{\substack{\tilde \pi(i)<\tilde \pi(j) , \\ \pi^*(i) > \pi^*(j)}} A_{i,j} \,. 
\end{equation}
This inequality is crucial to proving that $\hat \pi$ is close to $\pi^*$ with high probability.

To set up the rest of the proof, we define, for $\pi \in \cP$,
\begin{align*}
L_\pi &= |\{(i,j) \in [n]^2: \pi(i) < \pi(j) , \tilde \pi(i)>\tilde \pi(j) , \pi^*(i) > \pi^*(j) \}| \\
&= |\{(i,j) \in [n]^2: \pi(i) > \pi(j) , \tilde \pi(i)<\tilde \pi(j) , \pi^*(i) < \pi^*(j) \} | \,.
\end{align*}
Moreover, define the random variables
$$ 
X_\pi = \sum_{\substack{\scriptscriptstyle \pi(i) < \pi(j), \\ \scriptscriptstyle \tilde \pi(i)>\tilde \pi(j), \\ \scriptscriptstyle \pi^*(i) > \pi^*(j)}} A_{i,j}\,, \quad 
 Y_\pi = \sum_{\substack{\scriptscriptstyle \pi(i) > \pi(j), \\ \scriptscriptstyle \tilde \pi(i)<\tilde \pi(j), \\ \scriptscriptstyle \pi^*(i) < \pi^*(j)}} A_{i,j}\,, \quad \text{and} \quad Z = \sum_{\substack{\scriptscriptstyle \tilde \pi(i)<\tilde \pi(j), \\ \scriptscriptstyle \pi^*(i) > \pi^*(j)}} A_{i,j} . $$
We will prove that the random process $X_\pi - Y_\pi - Z$ is positive with high probability if $\pi$ is too far from $\tilde \pi$. However, \eqref{eq:key-ineq} says precisely that $X_{\hat \pi} - Y_{\hat \pi} - Z \le 0$, so that $\pi$ must be close to $\tilde \pi$ which is in turn close to $\pi^*$.

\paragraph{The case $M=\Mnss$ under sampling model \ref{model:1}.}
Consider model \ref{model:1} of sampling without replacement, and suppose that $M = \Mnss$ first. For a pair $(i,j)$ with $\pi^*(i) > \pi^*(j)$, the entry $A_{i,j}$ has distribution $\Ber \big(p(\frac 12 + \lambda) \big)$, since item $i$ and item $j$ are compared with probability $p$ and conditioned on them being compared, item $i$ wins with probability $\frac 12 + \lambda$. Moreover, $A_{i,j}$ is independent from any other $A_{k,\ell}$ with $\pi^*(k) > \pi^*(\ell)$. Hence $X_\pi$ has distribution $\Bin \big(L_\pi, p(\frac 12 + \lambda) \big)$. Similarly, $Y_\pi$ has distribution $\Bin \big(L_\pi, p(\frac 12 - \lambda) \big)$, and $Z$ has distribution $\Bin \big(\DD, p(\frac 12 + \lambda) \big)$. 
Therefore, Lemma~\ref{lem:bin-tail} implies that
\begin{enumerate}
\item
$ \p \big( X_\pi \le L_\pi p ( \frac 12 + \frac 12 \lambda ) \big) \le \exp \big( -L_\pi p \lambda^2/8 \big)  $,
and
\item
$ \p \big( Y_\pi \ge L_\pi p ( \frac 12 - \frac 12 \lambda ) \big) \le \exp \big( -L_\pi p \lambda^2/8 \big) . $
\end{enumerate}
Then we have that
\begin{equation} \label{eq:x-y-bound}
\p (X_\pi - Y_\pi \le L_\pi p \lambda) \le 2 \exp \big(- L_\pi p \lambda^2/8 \big) \,.
\end{equation}

For an integer $r \in [ C \varphi , \binom{n}{2} ]$ where $C$ is a sufficiently large constant to be chosen, consider the slice $\cS_r = \{\pi \in \cP: L_\pi = r \}$. Note that if $\pi \in \cS_r$, then
\begin{align} 
\inver(\pi, \pi^*) \nonumber 
&= |\{(i,j): \hat \pi(i) < \hat \pi(j) , \pi^*(i) > \pi^*(j) \}| \nonumber \\
%&= |\{(i,j): \hat \pi(i) < \hat \pi(j) , \tilde \pi(i)>\tilde \pi(j) , \pi^*(i) > \pi^*(j) \}| + |\{(i,j): \hat \pi(i) < \hat \pi(j) , \tilde \pi(i) < \tilde \pi(j) , \pi^*(i) > \pi^*(j) \}| \\
& \le |\{(i,j): \hat \pi(i) < \hat \pi(j) , \tilde \pi(i)>\tilde \pi(j) , \pi^*(i) > \pi^*(j) \}| \nonumber \\
& \quad \ + |\{(i,j): \tilde \pi(i) < \tilde \pi(j) , \pi^*(i) > \pi^*(j) \}| \nonumber \\
& = L_\pi + \inver(\tilde \pi, \pi^*) \le r + \varphi \,. \label{eq:inv-npi}
\end{align}
Since $\cP$ is a $\varphi$-packing of $\mathfrak S_n$ and $\cS_r \subseteq \cP$, we see that $|\cS_r|$ is bounded by the $\varphi$-packing number 
%can be extended to a $\varphi$-packing 
of the ball $\cB(\pi^*, r+\varphi)$ in the Kendall tau distance.
%$\{ \sigma\in \mathfrak S_n: \inver(\sigma, \pi^*) \le r + \varphi\}$.
Therefore, Proposition~\ref{prop:ball-cover-pack} gves
\[ 
%\log |\cS_r| \le n \log \frac{ 2n + 6r + 6 \varphi}{\varphi} + 3n 
\log |\cS_r| \le n \log \frac{ 2n + 2r + 2 \varphi}{\varphi} + 2n 
%\le 3 n \log \frac{14 r}{\varphi} \,.
\le n \log \frac{45 r}{\varphi} \,.
\]
By \eqref{eq:x-y-bound} and a union bound over $\cS_r$, we see that $\min_{\pi \in \cS_r} (X_\pi - Y_\pi) > c L_\pi p$
with probability at least
\begin{align*} 
& \quad \ 1- \exp \Big( n \log \frac{45r}{\varphi} + \log 2 - \frac{r p \lambda^2}8 \Big)
%& \quad \ 1- \exp \Big( 3 n \log \frac{14 r}{\varphi} + \log 2 - \frac{r p \lambda^2}8 \Big)
\\
&= 1- \exp \Big( n \log  \frac{45r}{\varphi} + \log 2 - \frac{r n}{8 \varphi} \Big) 
%&= 1- \exp \Big( 3 n \log  \frac{14r}{\varphi} + \log 2 - \frac{r n}{8 \varphi} \Big) 
\ge 1- \exp (- 2 n )
\,,
\end{align*}
where the inequality holds because $r/\varphi \ge C$ for a sufficiently large constant $C$. Then a union bound over integers $r \in [ C \varphi , \binom{n}{2} ]$ yields that $X_\pi - Y_\pi > c L_\pi p$ for all $\pi \in \cP$ such that $L_\pi \ge C \varphi$ with probability at least $1- e^{-n}$.

Furthermore, since $Z \sim \Bin \big(\DD, p(\frac 12 + \lambda) \big)$ and $\DD \le \varphi$, Lemma~\ref{lem:bin-tail} gives that
$$ \p (Z \ge 2 \varphi p )
\le \exp(- \varphi p / 4 ) \le \exp(-n/4) \,. $$ 
Combining the bounds on $X_\pi - Y_\pi$ and $Z$, we conclude that with probability at least $1 - e^{-n/8}$,
\[ X_\pi - Y_\pi - Z > c C \varphi p - 2 \varphi p > 0 \] 
for all $\pi \in \cP$ with $L_\pi \ge C \varphi$, as long as  $C>2/c$.

We have seen in \eqref{eq:key-ineq} that $X_{\hat \pi} - Y_{\hat \pi} - Z \le 0$, so $L_{\hat \pi} \le C \varphi$ on the above event. By \eqref{eq:inv-npi}, $\inver(\hat \pi, \pi^*) \le L_{\hat \pi} + \varphi$ on the same event, which completes the proof for the model \ref{model:1}.

\paragraph{The general case under sampling model \ref{model:1}.}
Let us continue to use $X_\pi$, $Y_\pi$ and $Z$ to denote the above random variables under the noisy sorting model $\cP$ with probability matrix $\Mnss$, and use $\tilde X_\pi$, $\tilde Y_\pi$ and $\tilde Z$ to denote the corresponding random variables under a general noisy sorting model $\tilde \cP$ with $M \in \Mns$. We couple the two models such that:
\begin{enumerate}
\item The sets of pairs of items being compared are the same (and if a pair is compared multiple times, the multiplicity is also the same);
\item For each pair $(i,j)$ with $\pi^*(i) >  \pi^*(j)$, if item $i$ beats item $j$ in a comparison in the model $\cP$, then it also beats item $j$ in the corresponding comparison in the model $\tilde \cP$.
\end{enumerate}
The second statement can be satisfied because the results of comparisons are Bernoulli random variables and $M_{\pi^*(i),\pi^*(j)} \ge [\Mnss]_{\pi^*(i),\pi^*(j)}$ for all $\pi^*(i) > \pi^*(j)$, by definition. Under this coupling, we always have that $\tilde X_\pi \ge X_\pi$ and $\tilde Y_\pi \le Y_\pi$, so the above high probability lower bound on $X_\pi - Y_\pi$ also holds on $\tilde X_\pi - \tilde Y_\pi$.

Moreover, recall the definition $\tilde Z = \sum_{\substack{\scriptscriptstyle \tilde \pi(i)<\tilde \pi(j), \\ \scriptscriptstyle \pi^*(i) > \pi^*(j)}} A_{i,j} ,$ where $A_{i,j} \sim \Ber \big( p [\Mnss]_{\pi^*(i), \pi^*(j)} \big)$. Since $[\Mnss]_{\pi^*(i), \pi^*(j)} \in (0,1)$, we can couple a sequence of i.i.d. $B_{i,j} \sim \Ber(p)$ with the $A_{i,j}$'s in such a way that  $B_{i,j} = 1$ whenever $A_{i,j} = 1$. Define $W = \sum_{\substack{\scriptscriptstyle \tilde \pi(i)<\tilde \pi(j), \\ \scriptscriptstyle \pi^*(i) > \pi^*(j)}} B_{i,j}$. Then we see that $W \sim \Bin(\DD, p)$ and $W \ge \tilde Z$. Since $\DD \le \varphi$, Lemma~\ref{lem:bin-tail} gives
$$ \p (W \ge 2 \varphi p )
\le \exp(- \varphi p / 4 ) \le \exp(-n/4) \,. $$ 
Thus $\tilde Z$ is subject to the same high probability upper bound as $Z$. Therefore, the proof for the model $\cP$ also works to show the desired bound for the model $\tilde \cP$.

\paragraph{Sampling model \ref{model:2}.}
The proof for model \ref{model:2} of sampling with replacement is essentially the same, except the part of probability bounds where we assume $M =\Mnss$. We now demonstrate the differences in detail. For a single pairwise comparison sampled uniformly from the possible $\binom n2$ pairs, the probability that 
\begin{enumerate}
\item the chosen pair $(i,j)$ satisfies $\pi(i) < \pi(j)$, $\tilde \pi(i)>\tilde \pi(j)$ and $\pi^*(i) > \pi^*(j)$, \emph{and}
\item item $i$ wins the comparison,
\end{enumerate}
is equal to $L_\pi \binom{n}{2}^{-1} (\frac 12 + \lambda)$. By definition, $X_\pi$ is the number of times the above event happens if $N$ independent pairwise comparisons take place, so $X_\pi \sim \Bin \big( N, L_\pi \binom{n}{2}^{-1} (\frac 12 + \lambda) \big)$. Similarly, we have $Y_\pi \sim \Bin \big( N, L_\pi \binom{n}{2}^{-1} (\frac 12 - \lambda) \big)$ and $Z \sim \Bin \big( N, \DD \binom{n}{2}^{-1} (\frac 12 + \lambda) \big)$. Hence Lemma~\ref{lem:bin-tail} gives that
\begin{enumerate}
\item
$ \p \big( X_\pi \le L_\pi N \binom{n}{2}^{-1} ( \frac 12 + \frac 12 \lambda ) \big) \le \exp \big( -L_\pi N \binom{n}{2}^{-1} \lambda^2/8 \big) ,$
\item
$ \p \big( Y_\pi \ge L_\pi N \binom{n}{2}^{-1} ( \frac 12 - \frac 12 \lambda ) \big) \le \exp \big( -L_\pi N \binom{n}{2}^{-1} \lambda^2/8 \big) ,$ and
\item
$\p \big( Z \ge 2 \varphi N \binom{n}{2}^{-1} \big)
\le \exp \big(- \varphi N \binom{n}{2}^{-1} / 4 \big) \,. $
\end{enumerate}
Note that if we set $p = N \binom{n}{2}^{-1}$, then the tail bounds above are exactly the same as those for the model \ref{model:1}. Therefore, replacing $p$ by $N \binom{n}{2}^{-1}$ everywhere in the above proof, we then obtain the desired bound for the model \ref{model:2}.
\end{proof}

Next, we turn to the lower bounds. Let $\p_{\pi^*} = \p_{\pi^*,\Mnss}$
denote the probability distribution of the observations in the noisy sorting model with underlying permutation $\pi^* \in \mathfrak{S}_n$ and probability matrix $\Mnss$,
where $\lambda \in (0,\frac 12)$. We prove the following stronger statement which clearly implies the lower bounds in Theorem~\ref{thm:minimax}.

\begin{theorem} \label{thm:lower-inv}
For the sampling model \ref{model:1},  suppose we have $\lambda \in (0,\frac 12)$ and $p \in (0,1]$ such that $p \log \frac{1}{1-2\lambda} \le C$ for some constant $C > 0$. Then it holds that
\[ \min_{\tilde \pi} \max_{\pi^* \in \mathfrak S_n} \, \p_{\pi^*} \Big( \inver (\tilde \pi, \pi^*) \gtrsim \frac{n}{ p \lambda^{2}} \land \frac n {p \log \frac{1}{1-2\lambda}} \land n^2 \Big) \ge c  \,, \]
where the minimum is taken minimized over all permutation estimators $\tilde \pi \in \mathfrak S_n$ that are measurable with respect to the observations and $c$ is a universal positive constant.
Similarly, for the sampling model \ref{model:2}, if we have $N n^{-2} \log \frac{1}{1-2\lambda} \le C$, then it holds that
\[ \min_{\tilde \pi} \max_{\pi^* \in \mathfrak S_n} \, \p_{\pi^*} \Big( \inver (\tilde \pi, \pi^*) \gtrsim \frac{n^3} {N \lambda^{2}} \land \frac{n^3} {N \log \tfrac{1}{1-2\lambda}} \land n^2 \Big) \ge c  \,. \]
\end{theorem}

Compared to the lower bounds in Theorem~\ref{thm:minimax}, the above lower bounds hold in probability, weaken the condition that $\lambda$ is bounded away from $1/2$ and only require maximizing $\pi^*$ instead of both $\pi^*$ and $M$, and are therefore stronger.

One key ingredient in proving lower bounds is to relate the Kullback-Leibler divergence between model distributions to the distance measuring the error~\citep[see, e.g.,][Chapter~2]{Tsy09}. This is achieved in the following lemma for both sampling models.

\begin{lemma} \label{lem:inv-kl}
Fix $\pi, \sigma \in \mathfrak S_n$ and $\lambda \in (0,\frac 12)$.
We denote by $\p_\pi$ the probability distribution of the noisy sorting model with underlying permutation $\pi$. Then for the sampling model \ref{model:1} we have 
$$ \KL (\p_\pi \| \p_\sigma) = 2 \, \inver(\pi, \sigma) \, p  \lambda \log \frac{1+2 \lambda}{1- 2\lambda}  \,, $$
and for the sampling model \ref{model:2} we have
$$ \KL (\p_\pi \| \p_\sigma) = 2\, \inver(\pi, \sigma) \, N \binom{n}{2}^{-1} \lambda \log \frac{1+2 \lambda}{1- 2\lambda} \,. $$
%\[  \frac{(a-b)^2}{a} \frac mn \inver(\pi, \sigma)  \le  \DD (\p_\pi \| \p_\sigma) \le 2  \frac{(a-b)^2}{b} \frac mn \inver(\pi, \sigma)  \,.  \]
\end{lemma}

\begin{proof}
First, we consider model \ref{model:1} of sampling without replacement. For $i \ne j$, let $\p_\pi^{(i,j)}$ denote the distribution of outcomes between $i$ and $j$, or more formally, the distribution of $N_{i,j}$ and $A_{i,j}$. For a pair $(i,j)$ such that $\pi(i) > \pi(j)$ and $\sigma(i) > \sigma(j)$, the distributions $\p_\pi^{(i,j)}$ and $\p_\sigma^{(i,j)}$ are indistinguishable. For $(i,j)$ such that $\pi(i) > \pi(j)$ and $\sigma(i) < \sigma(j)$, the probability that $i$ and $j$ are not compared stays the same, but the probability that they are compared and $i$ wins the comparison is $p(\frac 12 + \lambda)$ under $\p_\pi^{(i,j)}$ while it is $p(\frac 12 - \lambda)$ under $\p_\sigma^{(i,j)}$. A symmetric statement holds for the probability that they are compared and $j$ wins the comparison. Therefore, we obtain that
\begin{align*}
\KL (\p_\pi^{(i,j)} \| \p_\sigma^{(i,j)}) &= p (1/2 + \lambda) \log \frac{1/2 + \lambda}{1/2 - \lambda} + p (1/2 - \lambda) \log \frac{1/2 - \lambda}{1/2 + \lambda} \\
&= 2p  \lambda \log \frac{1+2 \lambda}{1- 2\lambda} \,.
\end{align*}
It follows from the chain rule that
%between Cartesian products of independent random variables is equal to the sum of the KL divergences between pairs of individual random variables (which are $\Ber(p)$ or $\Ber(q)$ in this case), 
%it follows that
\[ \KL (\p_\pi \| \p_\sigma) = \sum_{\pi(i)>\pi(j),\, \sigma(i)<\sigma(j)} \KL (\p_\pi^{i,j} \| \p_\sigma^{i,j}) 
= 2 \, \inver(\pi, \sigma) \, p  \lambda \log \frac{1+2 \lambda}{1- 2\lambda}  \,, \]
which proves the claimed bound.

Next, we move on to model \ref{model:2} of sampling with replacement. In this case, for the noisy sorting model with underlying permutation $\pi$, we let $\mathbb{Q}_\pi$ denote the distribution of the outcome of a single pairwise comparison chosen uniformly from the $\binom n2$ possible pairs. Conditioned on a pair $(i,j)$ with $\pi(i) > \pi(j)$ and $\sigma(i) > \sigma(j)$ being chosen, the outcome is indistinguishable under $\mathbb{Q}_\pi$ and $\mathbb{Q}_\sigma$. On the other hand, conditioned on having chosen $(i,j)$ with $\pi(i) > \pi(j)$ and $\sigma(i) < \sigma(j)$, the probability that $i$ wins the comparison is $p(\frac 12 + \lambda)$ under $\mathbb{Q}_\pi$ and is $p(\frac 12 - \lambda)$ under $\mathbb{Q}_\sigma$. By the definition of the KL divergence, we have
\begin{align*} 
\quad \ \KL( \mathbb{Q}_\pi \| \mathbb{Q}_\sigma ) 
&= \sum_{\pi(i)>\pi(j),\, \sigma(i)<\sigma(j)} \Big[ \binom{n}{2}^{-1} (1/2 + \lambda) \log \frac{1/2 + \lambda}{1/2 - \lambda} \\
&\qquad \qquad \qquad \qquad \qquad + \binom{n}{2}^{-1} (1/2 - \lambda) \log \frac{1/2 - \lambda}{1/2 + \lambda} \Big] \\
& = 2\, \inver(\pi, \sigma) \binom{n}{2}^{-1} \lambda \log \frac{1+2 \lambda}{1- 2\lambda}  \,,
\end{align*}
where the bound holds similarly as above. Since $N$ independent pairwise comparisons are observed and the KL divergence tensorizes, the conclusion follows.
\end{proof}

We are ready to prove the minimax lower bound.

\medskip

\begin{proof}[of Theorem~\ref{thm:lower-inv}]
Consider the sampling model \ref{model:1}. We assume that $n$ is lower bounded by a constant, and use the shorthand notation $\kappa = 4 p \lambda \log \frac{1+2 \lambda}{1- 2\lambda}$. Note that $\kappa \le C$ for some constant $C>0$ by the assumption. Let $r = c_0 n \kappa^{-1} \land \binom n2$ and $\eps = c_1 r$, where $c_0$ and $c_1$ are constants to be chosen. 
Let $\cP$ be a maximal $\eps$-packing of $\cB(\mathsf{id}, r)$, which is thus an $\eps$-net by maximality. For any $\pi, \sigma \in \cP$, we have $\inver(\pi, \sigma) \le 2r$, so Lemma~\ref{lem:inv-kl} yields 
\[ \KL(\p_\pi \| \p_\sigma) = \frac12 \, \kappa \, \inver(\pi, \sigma) \le \kappa r \le c_0 n  \,. \]

On one hand, if $\kappa \le c_2$ for a sufficiently small constant $c_2>0$, then $r \ge c_0 c_2^{-1} n \land \binom n2$ and thus Proposition~\ref{prop:ball-cover-pack} implies that
\[ 
\log |\cP| \ge n \log \frac{r}{n+\eps} - 2n 
%\log |\cP| \ge n \log \frac{r}{2n+2\eps} - 3n 
%\ge \frac 12 n \log \frac{r}{2n+4\eps} 
\ge 10 \, c_0 n \ge 10 \, \KL(\p_\pi \| \p_\sigma) \,,
\]
where we take $c_0 = 1$ and $c_1, c_2$ small enough for the inequalities to hold. 

On the other hand, if $c_2 < \kappa \le C$, then we take $c_1 = 1/8$ and $c_0$ sufficiently small so that $r \le c_0 c_2^{-1} n <n/2$. Then we can apply Lemma~\ref{lem:pack-special} to obtain
$$ \log |\cP| \ge \frac r5 \log \frac nr 
\ge \frac{c_0 n}{5 C} \log \frac {c_2}{c_0}
\ge 10 \, c_0 n \ge 10 \, \KL(\p_\pi \| \p_\sigma) \,, $$
%$$ \frac {\kappa}{4 c_0} \ge \exp(80\kappa) $$
%$$ c_0 \le \frac {\kappa}{4} \exp(-80\kappa) $$
%$$ r \asymp n \exp(- p) (1-2\lambda^*) $$
where the second inequality holds since $c_0 C^{-1} n \le r \le c_0 c_2^{-1} n$ and the third inequality holds for $c_0$ small enough. 

In either case, we have $\KL(\p_\pi \| \p_\sigma) \le 0.1 \log |\cP|$. Therefore, using  \citet[Theorem~2.5]{Tsy09} yields the lower bound of order $r \asymp n \kappa^{-1} \land n^2$. Considering the limiting behavior of $\kappa$ as $\lambda \to 0$ and $\lambda \to \frac 12$ repectively, we see that $\kappa \lesssim p \lambda^2 \lor p \log \frac{1}{1-2\lambda}$, so the claimed lower bound follows.

For the sampling model \ref{model:2}, the same argument follows if we replace $p$ with $N \binom{n}{2}^{-1}$.
\end{proof}

\subsection{Proof of Theorem~\ref{thm:multistage}} \label{subs:pf-main-thm}

Without loss of generality, assume that $\pi^* = \id$ and $n$ is even to simplify the notation. We define a score
$$ s_i^* = \sum_{j\in [n] \setminus \{i\}} M_{i,j} = \lambda (2i-n-1) + (n-1)/2 $$ 
for each $i \in [n]$, which is simply the $i$-th row sum of $M$ minus $1/2$. Analogously, we define 
$$\hat s_i = \sum_{j=1}^{i-1} (\frac 12 + \hat \lambda) + \sum_{j=i+1}^n (\frac 12 - \hat \lambda) = \hat \lambda (2i-n-1) + (n-1)/2 $$ 
for each $i \in [n]$, which is a slightly perturbed version of $s^*_i$ due to the difference between $\lambda$ and $\hat \lambda$. The \MS\ algorithm is designed to refine estimates for the scores $s^*_i$ in multiple stages.

First, the estimator $\hat \lambda$ satisfies the following bound, which in particular implies that $\hat s_i$ is close to $s^*_i$.

\begin{lemma} \label{lem:lambda}
If $N \ge C n \log n$, then we have $|\hat \lambda - \lambda| \le C_0 \sqrt{N^{-1} \log n}$ with probability at least $1- n^{-8}$, where $C$ and $C_0$ are sufficiently large universal constants.
\end{lemma}

\begin{proof}
Consider a single pairwise comparison chosen uniformly from the $\binom{n}{2}$ pairs. The probability that item $i$ is in the pair and wins the comparison is equal to
$\big( \sum_{j\in [n] \setminus \{i\}} M_{i,j} \big) / \binom{n}{2} = s_i^*/ \binom{n}{2}. $ Thus the random variable $S_i = \sum_{j =1}^n A_{i,j}'$ has distribution $\Bin \big(N/2, s_i^*/ \binom{n}{2} \big)$. Hence Lemma~\ref{lem:bin-tail} implies that
$$
\p\big( \big| S_i - \E[S_i] \big| \ge c_1 \E[S_i] \big) \le 2 \exp \big( - c_2 \E[S_i] \big) \le n^{-10} ,
$$
where the last inequality holds since $N \ge C n \log n$, and we use $c_1, c_2, \dots$ to denote sufficiently small constants.
A union bound shows that with probability at least $1-n^{-9}$, we have $| S_i - \E[S_i] | \le c_1 \E[S_i]$ for all $i \in [n]$. Denote this high probability event by $\cE$, and we condition on $\cE$ henceforth.

Recall that $s_i^* =  2\lambda i - \lambda(n+1) + (n-1)/2$. Using that $\lambda$ is bounded away from zero, we can choose $c_1$ small enough so that if $i-j \ge n/4$, then $s_i^* - s_j^* > 2 c_1 s_i^*$. Note that $E[S_i] = \frac 12 N s_i^*/ \binom n2$, so $E[S_i] - E[S_j] > 2 c_1 E[S_i]$ if $i-j \ge n/4$. Therefore, on the event $\cE$ we have $S_i > S_j$ for all $(i,j)$ with $i-j \ge n/4$. It follows that $\tilde \pi(i) > \tilde \pi(j)$ for these pairs $(i,j)$, as $\tilde \pi$ is defined by sorting the scores $S_i$.

Next consider $(i,j)$ such that $\tilde \pi (i) - \tilde \pi(j) > n/2$. Suppose we have $i<j$. Then there exists $k \in [n]$ with $\tilde \pi (j) < \tilde \pi(k) < \tilde \pi(i)$ such that either $k-i \ge n/4$ or $j-k \ge n/4$, which gives a contradiction on the event $\cE$. Therefore, it holds that $i > j$ for all pairs $(i,j)$ with $\tilde \pi (i) - \tilde \pi(j) > n/2$. 

Recall that $\hat \lambda = \frac{2}{N} \binom{n}{2} \binom{n/2}{2}^{-1} \sum_{(i,j) \in \cI} A_{i,j}'' - \frac 12,$ where $\cI = \{(i,j)\in [n]^2: \tilde \pi(i) - \tilde \pi (j) > \frac n2\}$. Note that $A''$ is independent of $\cE$, on which we have $i>j$ for all $(i,j) \in \cI$. 
Similar to the argument at the beginning of the proof, the probability that a uniformly chosen pair falls in $\cI$ and $i$ wins the comparison is $(\frac 12 + \lambda) |\cI|/\binom n2$. Hence the random variable $X \defn \sum_{(i,j) \in \cI} A_{i,j}''$ has distribution $\Bin \big( N/2, (\frac 12 + \lambda) |\cI|/\binom n2 \big)$. It follows that $\E[\hat \lambda \cond \cE] = \lambda$ once we note that $|\cI| = \binom{n/2}2$.

Moreover, Lemma~\ref{lem:bin-tail} gives the bound
$$
\p\Big( \big| X - \E[X] \big| \ge C_2 \sqrt{N \log n} \, \Big| \, \cE \Big) \le 2 \exp ( - c_3 \log n ) \le n^{-9},
$$
and consequently $|\hat \lambda - \lambda| \le C_0 \sqrt{N^{-1} \log n}$ with probability at least $1-n^{-9}$ conditioned on the event $\cE$, where $C_2$ and $C_0$ are sufficiently large constants. A union bound then completes the proof.
\end{proof}

We condition on the high probability event of Lemma~\ref{lem:lambda} throughout the rest of the proof, so that $|\hat \lambda - \lambda| \le C_0 \sqrt{N^{-1} \log n}$ for a fixed constant $C_0>0$. In particular, $\hat \lambda$ is bounded away from zero by a universal constant since $\lambda$ is and $N \ge C n \log n$, and $\hat s_j < \hat s_i$ iff $j < i$.
We proceed with the following key lemma.

\begin{lemma} \label{lem:subset-concentrate}
Fix $t \in [T]$, $i \in [n]$ and $I \subseteq [n]$ with $i \in I$. Suppose that $|I| \ge C_1 \frac{n^2 T}{N} \log (nT)$ for a sufficiently large constant $C$. If we define
$$
S = \frac{Tn(n-1)}{2N} \sum_{j \in I} A^{(t)}_{i,j} + \sum_{j \in [n] \setminus I,\, j< i} \big(\frac 12 + \hat \lambda \big) + \sum_{j \in [n] \setminus I,\, j> i} \big(\frac 12 - \hat \lambda \big) \,,
$$
then it holds with probability at least $1-2(nT)^{-9}$ that
$$
|S- \hat s_i| \le (5+C_0) n \sqrt{ |I| T N^{-1} \log(nT) } \,.
$$
\end{lemma}

\begin{proof}
Consider a single pairwise comparison chosen uniformly from the $\binom{n}{2}$ pairs. The probability that the chosen pair consists of item $i$ and an item in $I \setminus \{i\}$, and that item $i$ wins the comparison, is equal to
$ q \defn \big( \sum_{j \in I \setminus \{i\}} M_{i,j} \big) / \binom{n}{2} . $ Thus the random variable $X \defn \sum_{j \in I} A^{(t)}_{i,j}$ has distribution $\Bin(N/T, q)$. In particular, $\E[X] = Nq/T = \frac{2N}{Tn(n-1)} \sum_{j \in I \setminus \{i\}} M_{i,j} $ and by Lemma~\ref{lem:bin-tail},
$$
\p\Big( \big| X - \E[X] \big| \ge \frac{ r N }{T} \Big) \le 2 \exp \Big( - \frac{ N r^2 }{ 2 T (q+r) } \Big) .
$$
Taking $r = 6 \sqrt{\frac{T q}{N} \log(nT) }$, we see that $r \le q$ using the assumption $|I| \ge C_1 \frac{n^2 T}{N} \log (nT)$, so
\begin{equation} \label{eq:X-dev}
\p \Big( \big| X - \E[X] \big| \ge 6 \sqrt{q N T^{-1} \log(nT) } \Big) \le 2 (nT)^{-9} \,.
\end{equation}

By the definitions of $S$ and $\hat s_i$, it is straightforward to verify that
$$
S- \hat s_i = \frac{T n(n-1)}{2N} (X - \E[X]) + \sum_{j \in I,\, j<i} (\lambda - \hat \lambda) + \sum_{j \in I,\, j>i} (\hat \lambda - \lambda) \,.
$$ 
Therefore, we obtain from \eqref{eq:X-dev}, the definition of $q$ and the fact $|I| \le n$ that
\begin{align*}
|S- \hat s_i| &\le 3 n(n-1) \sqrt{q T N^{-1} \log(nT) } + |I| \, |\hat \lambda - \lambda| \\
&\le 5 n \sqrt{|I| T N^{-1} \log(nT) } + C_0 |I| \sqrt{N^{-1} \log n} \\
& \le (5+C_0) n \sqrt{|I| T N^{-1} \log(nT) }
\end{align*}
with probability at least $1- 2 (nT)^{-9}$.
\end{proof}

To analyze the \MS\ algorithm, we apply Lemma~\ref{lem:subset-concentrate} inductively to each stage of the algorithm. Define $\cE^{(0)}$ to be the full event. As the inductive hypothesis, we assume that on the event $\cE^{(t-1)}$, it holds that $j< i$ for all $j \in I_-^{(t-1)}(i)$ and $j>i$ for all $j \in I_+^{(t-1)}(i)$. In particular, this holds trivially for $t=1$. 

On the event $\cE^{(t-1)}$, the score $S^{(t)}_i$ is exactly the quantity $S$ in Lemma~\ref{lem:subset-concentrate} with $I = I^{(t-1)}(i)$.
%\begin{align*} 
%S^{(t)}_i &= \frac{Tn(n-1)}{2N} \sum_{j \in I^{(t-1)}(i) } A^{(t)}_{i,j} + \big(\frac 12 + \lambda^* \big) |I^{(t-1)}_-(i)| + \big(\frac 12 - \lambda^* \big) |I^{(t-1)}_+(i)| \\
%& \quad \ + \big(\hat \lambda - \lambda^* \big) \big( |I^{(t-1)}_-(i)| - |I^{(t-1)}_+(i)| \big) \\
%&= \frac{Tn(n-1)}{2N} \sum_{j \in I^{(t-1)}(i)} A^{(t)}_{i,j} + \sum_{j \in [n] \setminus I^{(t-1)}(i) } M_{i,j} + \big(\hat \lambda - \lambda^* \big) \big( |I^{(t-1)}_-(i)| - |I^{(t-1)}_+(i)| \big)
%\end{align*}
Thus the lemma shows that if $|I^{(t-1)}(i)| \ge C_1 \frac{n^2 T}{N} \log (nT)$ for a large enough constant $C_1$, then
\begin{equation} \label{eq:score-dev}
|S^{(t)}_i - \hat s_i| \le (5+C_0) n \sqrt{|I^{(t-1)}(i)| T N^{-1} \log(nT) } = \tau_i^{(t)}/2
\end{equation}
with probability at least $1-2(nT)^{-9}$ conditional on $\cE^{(t-1)}$. We denote by $\cE^{(t)}$ the sub-event of $\cE^{(t-1)}$ that the above bound holds for all $i \in [n]$. Then $\p(\cE^{(t)} \cond \cE^{(t-1)}) \ge 1- (nT)^{-8}$ and we condition on $\cE^{(t)}$ henceforth. 
%Then it holds that
%$$
%|S^{(t)}_i - s_i^*| \le 5 n \sqrt{|I^{(t-1)}(i)| T N^{-1} \log(nT) } + n |\hat \lambda - \lambda^*| = \tau_i^{(t)}/2 \,,
%$$
%since the term $n |\hat \lambda - \lambda^*|$ clearly bounds the difference between $S^{(t)}_i$ and $S$.

For any $j \in I^{(t)}_-(i)$, by definition $S^{(t)}_j - S^{(t)}_i < - \tau_i^{(t)}$, so we have $\hat s_j < \hat s_i$ and thus $j < i$. Similarly, $j> i$ for any $j \in I^{(t)}_+(i)$ on the event $\cE^{(t)}$. Hence the inductive hypothesis is verified. Moreover, note that $I^{(t)} (i) = \{j \in [n]: |S^{(t)}_j - S^{(t)}_i| \le 2 \tau_i^{(t)} \} \subseteq \{j \in [n]: |\hat s_j - \hat s_i| \le 3 \tau_i^{(t)} \} . $ Since $\hat s_j - \hat s_i = 2 \hat \lambda(j-i)$ and $\hat \lambda$ is bounded away from zero by a universal constant, we have
\begin{equation} \label{eq:iterate}
|I^{(t)} (i)| \le C_2 \tau_i^{(t)} = C_3 n \sqrt{|I^{(t-1)}(i)| T N^{-1} \log(nT) }  \,,
\end{equation}
where we use $C_2, C_3, \dots$ to denote sufficiently large constants. 
%By choosing the constants $C$ and $C_1$ in the assumptions $N \ge C n \log n$ and $|I^{(t-1)}(i)| \ge C_1 \frac{n^2 T}{N} \log (nT)$ to be large enough, we obtain
%\begin{equation} \label{eq:iterate}
%|I^{(t)} (i)| \le |I^{(t-1)}(i)| /2 \,.
%\end{equation}

Note that if we have $\alpha^{(0)} = n$ and the iterative relation $\alpha^{(t)} \le \beta \sqrt{\alpha^{(t-1)}}$ where $\alpha^{(t)} > 0$ and $\beta > 0$, then it is easily seen that $\alpha^{(t)} \le \beta^2 n^{2^{-t}}$. We would like to obtain such a bound from the relation \eqref{eq:iterate}. Note that $\cE^{(T)} \subseteq \cE^{(T-1)} \subseteq \cdots \subseteq \cE^{(0)}$ by definition and $\p(\cE^{(T)}) = \prod_{t=1}^T \p(\cE^{(t)} \cond \cE^{(t-1)}) \ge 1-n^{-8}$. Conditional on $\cE^{(T)}$, the iterative relation \eqref{eq:iterate} thus holds for all $t \in  [T]$, and we have $|I^{(0)}(i)| = n$ by definition. Since $I^{(t)} (i)$ is not updated in the algorithm once $|I^{(t)} (i)| \le C_1 \frac{n^2 T} N \log (nT)$, we obtain that
\begin{align*}
|I^{(T-1)} (i)| &\le \Big( C_3^2 \frac{n^2 T}{N} \log(nT) n^{2^{-T+1}} \Big) \lor \Big( C_1 \frac{n^2 T} N \log (nT) \Big) \\
& \le C_4 \frac{n^2}N (\log n) (\log \log n) \,,
\end{align*}
where the last bound holds because we take $T = \lfloor \log \log n \rfloor$.
Hence it follows from \eqref{eq:score-dev} that
$$
|S^{(T)}_i - \hat s_i| \le C_5 n^2 N^{-1} (\log n) (\log \log n) \,,
$$
and a similar argument as above shows that $S^{(T)}_i > S^{(T)}_j$ for all pairs $(i,j)$ with $i-j > C_6 n^2 N^{-1} (\log n) (\log \log n) = \, : \delta$. As the permutation $\hat \pi^{\scriptscriptstyle \MS}$ is defined by sorting the scores $S^{(T)}_i$ in increasing order, we see that $\hat \pi^{\scriptscriptstyle \MS}(i) > \hat \pi^{\scriptscriptstyle \MS}(j)$ for pairs $(i,j)$ with $i-j> \delta$.

Finally, suppose that $\hat \pi^{\scriptscriptstyle \MS} (i) - i < - \delta$ for some $i \in [n]$. Then there exists $j < i- \delta$ such that $\hat \pi^{\scriptscriptstyle \MS} (j) > \hat \pi^{\scriptscriptstyle \MS} (i)$, contradicting the guarantee we have just proved. A similar argument leads to a contradiction if $\hat \pi^{\scriptscriptstyle \MS} (i) - i > \delta$. Therefore, we obtain that
$$ |\hat \pi^{\scriptscriptstyle \MS} (i) - i| \le \delta = C_6 n^2 N^{-1} (\log n) (\log \log n) $$ 
for all $i \in [n]$, which completes the proof.

\section{Discussion and open problems} \label{sec:dis}
In this work, we focused on minimax estimation of the latent permutation $\pi^*$. Viewing $M = \frac 12 \bone_n \bone_n^\top$ as the null hypothesis and $M \in \Mns$ as the alternative hypothesis, a natural question is to establish the minimax detection level of the signal strength $\lambda$ in the hypothesis testing framework.

Moreover, we proved that the minimax rates for the noisy sorting problem do not involve any extra logarithmic factors even in the case of partial observations. For more complex models involving permutations \citep[see, e.g.][]{ColDal16, FlaMaoRig16, ShaBalGunWai17,PanWaiCou17,ShaBalWai17}, however, there are logarithmic gaps between current upper and lower bounds. According to the discussion after Proposition~\ref{prop:ball-cover-pack}, the logarithmic gaps do not necessarily stem from the unknown permutation, so it would be interesting to close these gaps or study whether they exist because of other aspects of the richer models.

For the \MS\ algorithm, it remains an open question whether analogous upper bounds can be established for  sampling without replacement. We conjecture that this is the case because of the empirical evidence in Section~\ref{sec:sim}. More importantly, there are still statistical-computational gaps unresolved for the general noisy sorting model where $M \in \Mns$, for the SST model of \cite{ShaBalGunWai17} and for the seriation model of \cite{FlaMaoRig16}. It would be interesting to know if the ideas behind the \MS\ algorithm could help tighten the gaps.

\section*{Acknowledgments.}
C.M. and P.R. were visiting the Simons Institute for the Theory of Computing while part of this work was done. C.M. thanks Martin J. Wainwright and Ashwin Pananjady for help discussions.
J.W. is supported in part by NSF Graduate Research Fellowship DGE-1122374. P.R. is supported in part by grants NSF DMS-1712596, NSF DMS-TRIPODS-1740751, DARPA W911NF-16-1-0551, ONR N00014-17-1-2147 and a grant from the MIT NEC Corporation.

\vskip 0.2in
\bibliography{sorting}

\end{document}